\documentclass[twoside]{article}
\usepackage{Packages}

\newcommand{\npt}[1]{\textcolor{blue}{[Nam: #1]}}

\newcommand{\andi}[1]{\textcolor{magenta}{[Andi: #1]}}

\begin{document}
\runningauthor{Nam Phuong Tran, Andi Nika, Goran Radanovic, Long Tran-Thanh, Debmalya Mandal}
\runningtitle{Sparse Offline Reinforcement Learning with Corruption Robustness}

\addtocontents{toc}{\protect\setcounter{tocdepth}{0}}

\twocolumn[
\aistatstitle{Sparse Offline Reinforcement Learning with Corruption Robustness}

\aistatsauthor{Nam Phuong Tran \And Andi Nika}
\aistatsaddress{ University of Warwick \And  MPI-SWS} 

\aistatsauthor{Goran Radanovic \And Long Tran-Thanh \And Debmalya Mandal}
\aistatsaddress{MPI-SWS \And University of Warwick \And University of Warwick} 
]

\begin{abstract}

   We investigate robustness to strong data corruption in offline sparse reinforcement learning (RL). In our setting, an adversary may arbitrarily perturb a fraction of the collected trajectories from a high-dimensional but \textit{sparse} Markov decision process, and our goal is to estimate a near optimal policy. The main challenge is that, in the high-dimensional regime where the number of samples $N$ is smaller than the feature dimension $d$, exploiting sparsity is essential for obtaining non-vacuous guarantees but has not been systematically studied in offline RL.
   We analyse the problem under \textit{uniform coverage} and \textit{sparse single-concentrability} assumptions. While Least Square Value Iteration (LSVI), a standard approach for robust offline RL, performs well under uniform coverage, we show that integrating sparsity into LSVI is unnatural, and its analysis may break down due to overly pessimistic bonuses. 
   To overcome this, we propose actor–critic methods with sparse robust estimator oracles, which avoid the use of pointwise pessimistic bonuses and provide the first non-vacuous guarantees for sparse offline RL under single-policy concentrability coverage. Moreover, we extend our results to the contaminated setting and show that our algorithm remains robust under strong contamination.
   Our results provide the first non-vacuous guarantees in high-dimensional sparse MDPs with single-policy concentrability coverage and corruption, showing that learning near-optimal policy remains possible in regimes where traditional robust offline RL techniques may fail.
\end{abstract}
\section{INTRODUCTION}

Offline reinforcement learning (RL) aims to learn effective decision-making policies solely from previously collected data, without further interaction with the environment. 
A central complication in practice is that real-world datasets can be corrupted, by logging errors, distribution shift, or even adversarial manipulation, so algorithms must be robust to a nontrivial fraction of contaminated trajectories. In this work we study corruption-robust offline RL under \emph{linear function approximation}, where both rewards and transitions admit linear models in a feature map.

Much of the existing offline RL theory analyses regimes with a modest number of features and shows that the sample size $N$ must scale polynomially with the ambient dimension $d$ to yield near-optimal policies. However, contemporary applications often operate with feature representations whose dimension $d$ far exceeds $N$, an increasingly common situation with deep models, rendering these guarantees vacuous or demanding prohibitively large datasets. To address this challenge, a common strategy is to assume the model exhibits a low-dimensional structure, with sparsity being a prominent form: only a small subset $S\subset[d]$ (with $s=|S|\ll d$) significantly influences the reward and transition probability. This setting is known as the \emph{sparse MDP} \citep{Golowich2023_SparseRL} and allows the sample size to scale polynomially with the sparsity level $s$ instead of the ambient dimension $d$.

Despite the importance of high-dimensionality, the sparse MDP literature, especially under corruption, remains under-explored. Existing positive results typically require strong coverage assumptions, such as \emph{uniform coverage}  \citep{Hao2020_OfflineSparseRL}. By contrast, realistic data may exhibit \emph{weak} coverage concentrated around a small set of policies. To clarify what is possible in such settings, we focus on the \emph{single-policy concentrability} regime, where data cover only one good policy (e.g., the optimal policy), and ask:


\emph{When $d > N$ and only single-policy concentrability coverage holds, can we exploit sparsity to learn a near-optimal policy? Moreover, can we extend these guarantees to the setting where part of the dataset is corrupted?}

In this paper, we tackle these questions and make significant progress in these directions, which can be summarised below.

\begin{enumerate}
\item \textbf{Difficulty of integrating sparsity into LSVI.}  
In sparse MDPs with \emph{single-policy concentrability coverage}, directly incorporating sparsity into the LSVI framework may fail, even without corruption. The core issue is that the pointwise pessimistic bonus, central to standard LSVI analysis, is incompatible with sparsity: without support knowledge, such bonuses cause excessive Bellman error and the guarantees break down.

\item \textbf{A sparsity-aware pessimistic Actor-Critic (AC) algorithm.}  
To address this limitation, we develop a \emph{pessimistic AC} framework that bypasses pointwise pessimistic bonuses. Under single-policy concentrability and no contamination, we obtain a suboptimality gap of order $\tilde O(H^2N^{-1/4}\sqrt{\kappa s})$, where $H$ is the horizon, $N$ the sample size, and $\kappa$ the relative condition number. 

\item \textbf{A sparsity-aware pessimistic AC with contamination.}  
We further extend our results to the contaminated setting. By integrating sparse robust regression oracles into the critic, our method achieves meaningful guarantees under both uniform coverage and single-policy concentrability, even when $d > N$ and an $\varepsilon$ fraction of trajectories are corrupted. In particular, we obtain a suboptimality gap of $O(\sqrt{s\varepsilon})$ with a statistically optimal but computationally expensive oracle, and $O(\sqrt{s}\,\varepsilon^{1/4})$ with a computationally efficient oracle.
\end{enumerate}

To the best of our knowledge, ours is the first result to achieve near-optimal policy learning in the high-dimensional regime $d > N$ under single-policy concentrability. Our findings also reveal a sharp contrast between two prominent offline RL paradigms: although both LSVI and AC methods can be near-optimal in \emph{non-sparse} MDPs, enforcing pessimism in a pointwise manner is unnatural for LSVI analysis and renders the guarantee vacuous in the \emph{sparse} setting, whereas the AC approach naturally accommodates sparsity and provides non-vacuous guarantees.
To save space, all proofs in this paper are deferred to the Appendix.

\subsection{Related works}
We now briefly outline related work and refer the reader to
Appendix \ref{appendix: related works} for a more in-depth literature review.

\begin{table*}[t]
\centering
\resizebox{\textwidth}{!}{%
\begin{tabular}{|l|c|c|c|c|c|}
\hline
 & Suboptimality gap & \#Sample required & \makecell{Non-vacuous bound \\ when $d > N$?} & Coverage assumption \\
\hline
\citep{ye2023corruption} & $\tilde O\!\left(\tfrac{H^3 \sqrt{d}\epsilon}{\xi}\right)$ & $\mathrm{poly}(d,H)$ & \xmark & Uniform coverage \\
\hline
\citep{Zhang2021_RobustOfflineLinMDP} & $\tilde O\!\left(\tfrac{H^3 \sqrt{d}\epsilon}{\xi}\right)$ & $\mathrm{poly}(d,H)$ & \xmark & Uniform coverage \\
\hline
\textbf{Ours} & $\tilde O\!\left(\tfrac{H^3 s\sqrt{\epsilon}}{\xi}\right)$ & $\mathrm{poly}(s,H)$ & \cmark & Uniform coverage \\
\hline
\citep{Zhang2021_RobustOfflineLinMDP} & $\tilde O\!\left(H^3 \sqrt{d\kappa\epsilon}\right)$ & $\mathrm{poly}(d,H)$ & \xmark & Single-policy concentrability \\
\hline
\textbf{Ours} & $\tilde O\!\left(H^3 \sqrt{s\kappa\epsilon}\right)$ & $\mathrm{poly}(s,H)$ & \cmark & Single-policy concentrability \\
\hline
\end{tabular}%
}


\begin{minipage}{\textwidth}
\small
\caption{Comparisons with existing results on corruption-robust offline RL.
The table compares our results with existing results in the literature on corruption-robust offline RL.
The column \emph{Non-vacuous bound} means that, in the clean setting, the suboptimality gap can approach $0$ when $d > N$.
Our result in this table is stated under the boundedness assumption $\|\phi(\cdot)\|_\infty \leq 1$ and $\max_h\max\{\|\theta_h\|_1,\|\mu_h(\mathcal X)\|_1\} \leq 1$ for fair comparison.
In the \emph{Coverage assumption} column, \emph{Uniform coverage} means that the covariance matrix has minimum eigenvalue at least $\xi$ (see Assumption~\ref{assp: uniform coverage}),
while \emph{Single-policy concentrability coverage} means that the relative condition number is at most $\kappa$ (see Assumption~\ref{assp: Relative condition number}).}
\label{table:comparison}
\end{minipage}
\end{table*}

\paragraph{Offline RL.} The central challenge in offline RL is to learn a near-optimal policy from data that provides only partial information about the environment. To address this challenge, value-iteration-based methods \citep{buckman2020importance, liu2020provably, kumar2019stabilizing, jin2020provably, yu2020mopo} and AC methods \citep{levine2020offline, wu2019behavior, zanette2021_provablebenefitsactorcriticmethods} are among the most popular approaches. We consider methods that apply pessimism by penalising value functions of policies that are under-represented in the data. However, as we show in Section \ref{subsec: Sparse Robust LSVI with Single Concentrability}, such an approach may lead to vacuous bounds in the sparse setting, arguably due to an overcompensation for uncertainty arising from the use of pointwise pessimistic bonuses. This also motivates the use of AC methods, which we show can effectively mitigate such an issue.

\paragraph{Corruption-robust offline RL.} In this paper, we consider data poisoning attacks in offline RL \citep{Zhang2021_RobustOfflineLinMDP}. Previous work on corruption-robust offline RL has primarily considered the low-dimensional setting where $N \gg d$ \citep{Zhang2021_RobustOfflineLinMDP, ye2023corruption, mandal2025corruption}. In contrast, this work focuses on the high-dimensional setting $d > N$. In particular, we employ a value function estimation approach \citep{Zhang2021_RobustOfflineLinMDP} based on sparse robust estimation \citep{Merad2022_RobustSparseLinearRegressionHeavyTailDist}. We find that, under single-policy concentrability, in sparse settings, integrating sparse oracles into LSVI-type algorithms is considerably more challenging and may yield suboptimal bounds that depend polynomially on $d$, in contrast to prior work in the low-dimensional setting \citep{Zhang2021_RobustOfflineLinMDP}. Moreover, we demonstrate that sparse robust estimation can be naturally integrated into the AC framework, effectively eliminating polynomial dependence on $d$.
A systematic comparison of our results with previous work on corruption-robust offline RL is provided in Table \ref{table:comparison}.

\paragraph{Sparse Linear MDP.} 
Sparse linear MDPs have been a primary focus in the online RL literature~\citep{Golowich2023_SparseRL,Kim2023_SparseRL,Hao2020_OfflineSparseRL}, where one can often design exploratory policies to ensure good data coverage and thereby exploit the underlying sparsity structure. In contrast, there are far fewer works on sparse MDPs in the offline RL setting \citep{Hao2020_OfflineSparseRL}, where strong data coverage assumptions (e.g., uniform coverage) are typically imposed to bypass the need for pessimism, as we shall explain later in this paper. However, we note that \emph{limited data coverage} is the central challenge of offline RL compared to online RL, without an exploratory policy that guarantees sufficient coverage, it is unclear from prior work whether one can learn a near-optimal policy. For this reason, our paper primarily focuses on the case of limited coverage, namely \emph{single-policy concentrability}, and is the first to establish non-vacuous suboptimality guarantees for sparse MDPs in this regime.

\section{PRELIMINARIES}
\paragraph{Notations.} For $K\in \mathbb N^+$, denote $[K] = \{1,\dots,K \}$.
Let $\mathbb I$ be the indicator function.
For $z \in \mathbb{R}^d$ and $S \subseteq [d]$, let $z_S \in \mathbb{R}^d$ be the vector 
obtained by zeroing out coordinates outside $S$, i.e. $
(z_S)_i = z_i \mathbb I\{i \in S\}$.
For any symmetric matrix $\Sigma \in \mathbb{R}^{d \times d}$ and index set $S \subseteq [d]$, 
let $\Sigma_S \in \mathbb{R}^{|S|\times|S|}$ denote the principal submatrix of $\Sigma$  indexed by $S$, i.e.,
$
\Sigma_S = \big[\Sigma_{ij}\big]_{i \in S,\, j \in S}.
$
For a symmetric matrix $\Sigma$, denote $\lambda_{\min}(\Sigma)$ as its smallest eigenvalue.

\paragraph{General MDPs.} Consider an episodic MDP $\mathcal M = (\mathcal X, \mathcal A, H, P, R, x_1 )$, where $\mathcal X$ is the state space, $\mathcal A$ is the action space, $H$ is the episode length, $P=\{P_h: \mathcal X \times \mathcal A \rightarrow \Delta(\mathcal X)\}_{h=1}^H$ denotes the transition probabilities (here, $\Delta(\mathcal X)$ denotes the space of probability distributions over $\mathcal X$), $R$ is a bounded stochastic reward with support $\mathrm{supp}(R_h) = [0, 1]$ and mean $r_h:\mathcal X \times \mathcal A \rightarrow [0,1]$ for all $h\in [H]$, $x_1$ is the initial state.
For a policy $\pi$, let $d^\pi = (d^\pi_h)_{h=1}^H$ denote the occupancy measures over state-action pairs induced by $\pi$ and the transition probability $P$, that is, $d^\pi_h = \mathbb P(x_h =x, a_h=a\mid \pi, x_1)$.

For any function $f: \mathcal X \times \mathcal A \rightarrow \mathbb R$, we define the $\pi$-Bellman operator and the (optimal) Bellman operator
\begin{equation*}
\begin{aligned}
    (\Bell^\pi_h f)(x,a) &\triangleq r_h(x,a) 
    \; + \underset{\substack{x'\sim P_h(\cdot \mid x,a) \\ a' \sim \pi(\cdot\mid x')}}{\mathbb E} \SquareBr{f(x',a')}; \\
    (\Bell_h f)(x,a) &\triangleq r_h(x,a) + \underset{x'\sim P_h(\cdot \mid x,a)}{\mathbb E} \SquareBr{\max_{a'\in \mathcal A} f(x',a')}.
    \end{aligned}
\end{equation*}
We then have the Bellman equation
\begin{equation*}
\begin{aligned}  
    Q^\pi_h(x,a) = (\Bell_h^\pi Q^\pi_{h+1})(x,a),\,
    V^\pi_h(x) = \underset{a\sim \pi(\cdot|x)}{\mathbb E}[Q^\pi_h(x,a)],
\end{aligned}
\end{equation*}
and the Bellman optimality equation
\begin{equation*}
    Q^\star_h(x,a) = (\Bell_h Q^\star_{h+1})(x,a),\,
     V^\star_h(x) = \max_{a\in \mathcal A}Q^\star_h(x,a).
\end{equation*}
Let $\pi_\star$ be an optimal policy, for a policy $\pi$, define the sub-optimal gap
\[
\mathrm{SubOpt}(\pi_\star,\pi) \triangleq V^{\pi_\star}(x_1) - V^{\pi}(x_1).
\]

\begin{assumption}[Sparse linear MDP]\label{assp: sparse linear mdp}
There exists a feature map $\phi: \mathcal X \times \mathcal A \rightarrow \mathbb R^d$, signed measures $\mu_h: \mathcal X \rightarrow \mathbb R^d$, and parameter vectors $\theta_h \in \mathbb R^d$ such that, for all $(x', x, a)$, 
\begin{equation}\label{eq:linear-MDP}
    \begin{aligned}
        P_h(x'\mid  x, a) &= \langle \phi(x,a), \mu_h(x') \rangle, \\
        r_h(x,a) &  = \langle \phi(x,a), \theta_h \rangle.
    \end{aligned}
\end{equation}
We assume that the MDP is $s$-sparse, that is, there is $S \subset [d], \; |S| = s$, such that for any $i \notin S$, $(\mu_{h}(x))_i = 0$ for all $x \in \mathcal X$, and $\theta_{i} = 0$. 
We assume that for all $(x,a,h) \in \mathcal X \times \mathcal A \times [H]$, $\|\phi(x,a)\|_\infty \leq 1$, $\|\theta_h\|_1 \leq 1$, and  $\|\mu_h(\mathcal X)\|_1 \leq 1$.
\end{assumption}

Next, we can define the space of linear Q-function as follows: for any $h$, define
\begin{align*}
    \mathcal Q_{h} = \{&(x,a)\mapsto \langle \phi(x,a),  w\rangle \mid \\
    & \quad w \in \mathbb R^d,\; \| w \|_1\leq H+1-h,\; \| w \|_0\leq s\}.
\end{align*}

\paragraph{Offline data set.} 
We consider the offline setting with contaminated data.
Let $\tilde{\mathcal D}$ be a clean dataset consisting of $N$ trajectories, i.e., $\tilde{\mathcal D} = \{(x^\tau_h, a_h^\tau, R_h^\tau)\}_{\tau=1, h=1}^{N,H}$.
Suppose there is an adversary who observes the dataset $\tilde{\mathcal D}$ and may arbitrarily corrupt up to $\epsilon N$ trajectories, resulting in a corrupted dataset $\mathcal D$.
We split the dataset into \(H\) disjoint subsets
\(\mathcal D_1,\ldots,\mathcal D_H\), and use \(\mathcal D_h\) only for the
stage-\(h\) regression.
In this work, regarding clean data set, we assume there is an exploratory policy $\pi_\nu$, with the occupancy probability $\nu = (\nu_h)_{h=1}^H$.
Define $\Sigma_h \triangleq \mathbb E_{\nu_h}\SquareBr{\phi(x,a)\phi(x,a)^\top}$ as the covariance matrix of the underlying data distribution $\nu$ at horizon $h$.  
We introduce an assumption on the distribution $\nu$, which is crucial for our main result.

\begin{assumption}[Uniform coverage] \label{assp: uniform coverage}
We say that the data distribution $\nu$ is $\xi$-covering the state–action space with $\xi > 0$ if, for any $h \in [H]$, 
$
\Sigma_h \succeq \xi I,
$
i.e., the smallest eigenvalue of $\Sigma_h$ is at least $\xi$.  
\end{assumption}

\begin{assumption}[Sparse single-policy concentrability]\label{assp: Relative condition number}
Let $\Sigma_{\star,h} = \Sigma_{\pi_\star,h}$ denote the covariance matrix induced by the optimal policy $\pi_\star$ at horizon $h$.  
For any $h \in [H]$ and vector $z \in \mathbb R^d$ with $\|z \|_0 \leq 2s$, assume that  
\begin{equation}
    \max_{z \in \mathbb R^d:\|z\|_0\leq 2s }\frac{z^\top \Sigma_{\star,h} z}{z^\top \Sigma_h z} \leq \kappa.
\end{equation}
\end{assumption}
\section{SPARSE ROBUST LINEAR REGRESSION}
\label{Section: Sparse Robust Linear Regression}
Suppose there is a collection of data points $
\tilde{\mathcal D}= \{(z_i, y_i)\}_{i=1}^N
$
generated i.i.d. from the underlying model
$
y = z^\top  w_\star + \eta,
$
where $z \sim \nu \in \Delta(\mathbb R^d)$ satisfies $\|z\|_\infty \leq 1$.  
Denote the covariance matrix as 
$
\Sigma = \mathbb E_{z\sim \nu}[z z^\top].
$
The unknown parameter $ w_\star \in \mathbb R^d$ is an $s$-sparse vector, and $\eta$ is zero-mean sub-Gaussian noise with $\mathrm{Var}(\eta) \leq \sigma^2$.  
Suppose an adversary can arbitrarily corrupt up to $\epsilon N$ data points, resulting in a corrupted dataset $\mathcal D$, and the estimator has access only to $\mathcal D$.
We now discuss the estimation error of different robust sparse linear estimators under various settings of $\Sigma$.  

Given a dataset $\mathcal{D}$, the Sparse Robust Linear Estimator, denoted \texttt{\textbf{SRLE}}, is a mapping that takes $\mathcal{D}$ as input and outputs an estimator $\hat{w}$ of the true parameter $w_\star$, i.e., $\hat{w} = \texttt{\textbf{SRLE}}(\mathcal{D})$.
In this section, we will investigate different types of \texttt{\textbf{SRLE}} under different assumptions on the covariate distribution $\nu$.

\subsection{Robust Sparse Linear Regression with Uniform Coverage} \label{Sec: Robust Sparse Linear Regression with Uniform Coverage}

We begin by analyzing the case where the covariate matrix $\Sigma$ is well-conditioned, i.e., its minimum eigenvalue is bounded away from zero. 
This \emph{uniform coverage} assumption ensures that the features are sufficiently informative across all directions, which allows us to design computationally efficient and statistically robust estimators. 
The following result, adapted from \citep{Merad2022_RobustSparseLinearRegressionHeavyTailDist}, provides guarantees for such an estimator.

\begin{restatable}{proposition}{PropRSLSwithDataCoverage}\label{prop: RSLS with data coverage}
If $\lambda_{\min}(\Sigma) \geq \xi > 0$, then with probability at least $1-\delta$, there exists a robust least squares estimator, named \texttt{SRLE1}, which returns an estimate $\hat w$ satisfying
\begin{equation}
    \|\hat  w -  w_\star\|_1 = O\RoundBr{(\sigma + \| w_\star\|_1)\RoundBr{\frac{\, s \log(d/\delta)}{\xi \sqrt{N}} + \frac{ s \sqrt{\epsilon}}{\xi}}}.
\end{equation}
Moreover, the estimator runs in $\mathrm{poly}(d,s,N)$ time.
\end{restatable}
To the best of our knowledge, \texttt{SRLE1} is a practical estimator that achieves a small error bound when the covariate distribution is sub-Gaussian, although we suspect its guarantee is highly suboptimal. Since developing robust estimators is not the main focus of this paper, but rather understanding how to use them effectively in offline RL, further advances in sparse regression under corruption could directly strengthen our results by replacing the \texttt{SRLE1} oracle with a stronger alternative.

Proposition \ref{prop: RSLS with data coverage} shows that, under uniform coverage, we can achieve both robustness to data contamination and computational efficiency. 
However, in practice the uniform coverage assumption may not hold, especially in high-dimensional settings where the covariate matrix can be ill-conditioned, even when restricted to sparse subspaces. 
We therefore turn to the more challenging case where no such assumption is imposed. 

\subsection{Robust Sparse Linear Regression without Uniform Coverage}
When $\Sigma$ may be ill-conditioned, that is, $\lambda_{\min}(\Sigma) \rightarrow 0$, the estimation error bound in Proposition \ref{prop: RSLS with data coverage} becomes vacuous. 
In this setting, we first present a statistically optimal but computationally expensive estimator, and then show another computationally efficient estimator with larger statistical error.
Due to the lack of uniform coverage, the estimation error in terms of $\|\cdot\|_1$ can no longer be guaranteed. Instead, we provide estimation error bounds with respect to the $\Sigma$-norm $\|\cdot \|_\Sigma$, which will be sufficient for the RL algorithm used later.

\paragraph{A non-computationally efficient estimator.}
The following theorem establishes the existence of an estimator, denoted \texttt{SRLE2}, which achieves minimax-optimal statistical guarantees even without uniform coverage.

\begin{restatable}{proposition}{ThmRSLSwithoutDataCoverage} \label{theorem: RSLS without data coverage}
For any covariate matrix $\Sigma$, with probability at least $1-\delta$,there exists an estimator named \texttt{SRLE2} returns $\hat  w$ such that
\begin{equation}
      \|\hat  w -  w_\star \|_{\Sigma}^2 = O\RoundBr{\frac{\sigma\| w_\star \|_2\sqrt{s}\log(d/\delta)}{\sqrt{N}}  + \sigma^2\epsilon}. 
\end{equation}
\end{restatable}
Note that, in the case of an ill-conditioned covariate matrix, the minimax optimal rate is $1/\sqrt{N}$ \citep{Hsu2014_RLS_GeneralRandomDesign}, which is known as the slow rate in the literature.
We also note that previous work such as \citep{Jin2020_provableLSVI, Zhang2021_RobustOfflineLinMDP} uses an oracle with rate $1/N$, but with a hidden constant that scales inversely with $\lambda_{\min}(\Sigma)$, which is invalid in our setting as $\lambda_{\min}(\Sigma) \to 0$.

\paragraph{A computationally efficient estimator.} 
While \texttt{SRLE2} enjoys optimal statistical guarantees of order $O(\epsilon)$, it is computationally intractable in general (see Appendix \ref{appendix: SRLE2} for details), since best subset selection is a NP-hard problem. 
To address this limitation, one can employ an alternative algorithm, \texttt{SRLE3}, which is polynomial-time computable at the cost of looser statistical guarantees.
\begin{restatable}{proposition}{PropRSLSwithoutDataCoverageCompEfficient} \label{prop: RSLE mirror descent ill-condition covariate}
  For any covariate matrix $\Sigma$, there exists an algorithm named \texttt{SRLE3} such that, with probability at least $1-\delta$, it returns an estimate $\hat  w$ satisfying
  \begin{equation*}
  \begin{aligned}
      \|\hat  w -  w_\star \|_{\Sigma}^2 = O\Biggl(&\| w_\star\|_1 (\| w_\star\|_1 + \sigma)\\
      &\times \RoundBr{\sqrt{\frac{\log(d\delta^{-1})}{N}} + \sqrt{\epsilon}} \Biggl).
  \end{aligned}
  \end{equation*}
  Moreover, the estimator runs in $\mathrm{poly}(N,d,s)$ time.
\end{restatable}
Note that the statistical error for \texttt{SRLE3} is $O(\sqrt{\epsilon})$, 
in contrast to $O(\epsilon)$ for the computationally expensive oracle \texttt{SRLE2}.
This trade-off highlights the fundamental tension in the absence of uniform coverage: 
one may either obtain statistically optimal but computationally expensive estimators, or efficient algorithms with degraded accuracy.
We also note that the bound in Proposition \ref{prop: RSLE mirror descent ill-condition covariate} can have implicit dependence on $s$ through $\| w_\star\|_1$.

\subsection{Using \texttt{\textbf{SRLE}} Estimators in Offline Reinforcement Learning}
So far, we have established guarantees for several variants of Sparse Robust Linear Estimators (\texttt{\textbf{SRLE}}) under different assumptions on the covariate distribution. 
While these regression results are of independent interest, their main role in this paper is to serve as a key component in the design and analysis of offline reinforcement learning algorithms. 
In particular, the \texttt{\textbf{SRLE}} estimators will be used to approximate the linear predictors that arise when estimating value functions.

Formally, recall that for each horizon $h$, the population covariance is
\[
\Sigma_h = \mathbb E_{\nu_h}\!\left[\phi(x,a)\phi(x,a)^\top\right],
\]
and its empirical analogue is defined as
\[
\widehat \Sigma_h \triangleq\frac{1}{|\mathcal D_{h}|}\sum_{(x,a)\in \mathcal D_h}\phi(x,a)\phi(x,a)^\top + (\lambda + \epsilon)I.
\]

Let $\mathcal F$ denote the space of real-valued functions on $\mathcal X \times \mathcal A$. 
Given a policy $\pi$, first we define a data set compatible with $\texttt{SRLE}$ using the offline data set $\mathcal D$ and policy $\pi$, as 
\begin{equation*}
    \mathcal D^\pi_h \triangleq \CurlyBr{ \RoundBr{\underbrace{\phi(x_h^\tau, a_h^\tau)}_{z_i}, \; \\
    \underbrace{R_h^\tau + \mathbb E_{\pi_{h+1}}\SquareBr{F(x^\tau_{h+1}, a)}}_{y_i}} }_{\tau=1}^{N/H}.
\end{equation*}
Next, we define three operators, a projection operator $\mathcal P^\pi$, a regression operator $\mathcal R^\pi$, and the associated estimation error $\mathcal E^\pi$, as mappings from $\mathcal F$ to $\mathbb R^d$:
\begin{equation}
\begin{aligned}
    \mathcal P_h^{\pi}(F) &\triangleq \theta_h + \int_{\mathcal X} \left(\sum_{a'\in \mathcal A}\pi(a'| x') 
     F(x',a') \right) \mu_h(x')\, dx', \\
    \mathcal R_h^\pi(F) &\triangleq \texttt{\textbf{SRLE}}(\mathcal D^{\pi}_h), \\
    \mathcal E^\pi_h(F) &\triangleq \mathcal P^\pi_h(F) - \mathcal R^\pi_h(F).
\end{aligned}
\end{equation}

In words, for any function $F$, the projection operator $\mathcal P^\pi_h(F)$ represents the best-fit linear predictor, the regression operator $\mathcal R^\pi_h(F)$ applies an \texttt{\textbf{SRLE}} estimator to approximate this predictor using corrupted data, and $\mathcal E^\pi_h(F)$ captures the resulting estimation error. 
Note that the precise statistical guarantees of $\mathcal R^\pi_h$ depend on the choice of \texttt{\textbf{SRLE}} variant (e.g., \texttt{SRLE1}, \texttt{SRLE2}, \texttt{SRLE3}), and we will specify this choice in each subsequent result.

For later use, we also define the greedy projection operator $\mathcal P^*$, greedy regression operator $\mathcal R^*$ and the associated estimation error $\mathcal E^*$ which replaces the policy expectation with a maximisation:
\begin{equation}
\begin{aligned}
 \mathcal D^*_h &\triangleq \CurlyBr{ \RoundBr{\phi(x_h^\tau, a_h^\tau), \; 
   R_h^\tau + \max_{a\in \mathcal A}F(x^\tau_{h+1}, a) }}_{\tau=1}^{N/H}. \\
\mathcal P_h^{*}(F) &\triangleq \theta_h + \int_{\mathcal X} \max_{a'\in \mathcal A} F(x',a') \mu_h(x')\, dx', \\
     \mathcal R_h^*(F) &\triangleq \texttt{\textbf{SRLE}}(\mathcal D^*_h), \\
     \mathcal E_h^*(F) & \triangleq \mathcal R_h^*(F) - \mathcal P_h^{*}(F).
\end{aligned}
\end{equation}

\section{WHY ROBUST LSVI MAY FAIL IN SPARSE OFFLINE RL}
\label{section: Why Robust LSVI may Fail in Sparse Offline RL}
In the offline RL literature, the Least Square Value Iteration (LSVI) framework is a popular choice \citep{Jin2020_provableLSVI}.  
Moreover, they can successfully handle data corruption, which makes them a strong candidate for addressing the problem studied in this paper \citep{Zhang2021_RobustOfflineLinMDP}.
However, in this section we show that while LSVI-type algorithms can succeed under the uniform coverage assumption (Assumption \ref{assp: uniform coverage}), they can fail when only the single concentrability assumption (Assumption \ref{assp: Relative condition number}) is imposed.
We begin by describing the Sparse Robust LSVI algorithm, whose pseudocode is provided in Algorithm \ref{alg: Sparse Robust LSVI}.

\begin{algorithm}
    \caption{Sparse Robust Least-Square Value Iteration} \label{alg: Sparse Robust LSVI}
    \begin{algorithmic}
        \STATE Input data set $\mathcal D$, \texttt{SRLE} oracle, Pessimistic bonus function $\{\Gamma_h\}_{h=1}^H$. \\
        \STATE Initialise $\underline Q_{H+1} = \bm 0$. \\
        \FOR{$h \in [H]$}
            \STATE Set $\hat w_h \leftarrow \mathcal R^*_h(\underline Q_{h+1})$. \\
            \STATE Set $\underline Q_h(x,a) = \phi(x,a)^\top \hat w_h + \Gamma_h(x,a)$, clipped in $[0, H+1-h]$. \\
            \STATE Set $\hat \pi_h(a\mid x) = \begin{cases}
                1 \text{ if } a = \argmax_a \underline Q_h(x,a) \\
                0 \text{ otherwise }.
            \end{cases}.$\\
        \ENDFOR
        \STATE Output $\{\hat \pi_h\}_{h=1}^H$.
    \end{algorithmic}
\end{algorithm}

\subsection{Sparse Robust LSVI with Uniform Coverage} \label{subsec: Sparse Robust LSVI with Uniform Coverage}
The uniform coverage assumption enables the use of \texttt{SRLE1} estimator in Proposition \ref{prop: RSLS with data coverage}, which in turn allows us to apply the pessimistic LSVI framework for robust offline RL.

\begin{restatable}{proposition}{PropSubOptUniformCoverageLSVI}\label{prop: SubOpt LSVI under Uniform Coverage}
    Suppose Assumption \ref{assp: sparse linear mdp} and \ref{assp: uniform coverage} hold.
    Let the pessimistic bonus be $\Gamma_h(x,a) = 0$ for all $h \in [H],\; x \in \mathcal X,\; a \in \mathcal A$.  
    Run Algorithm \ref{alg: Sparse Robust LSVI} with the \texttt{SRLE1} estimator.  
    Then, with probability at least $1-\delta$.
    \begin{equation}
        \mathrm{SubOpt}(\pi_\star, \hat \pi) =  O\left(\frac{H^{3}s \log(d/\delta)}{\xi \sqrt{N}} + \frac{H^3s \sqrt{\epsilon}}{\xi}\right).
    \end{equation}
\end{restatable}

\begin{remark}
Under the uniform coverage assumption, the parameter $\xi$ can be regarded as a constant (independent of $d$) in many settings, as long as the feature vectors satisfy $\|\phi(\cdot)\|_\infty \leq 1$. 
For example, when $\phi(x,a)$ is sampled from $\{-1,1\}^d$ under the uniform distribution, we have $\xi = 1$. 
Therefore, the key feature of Proposition \ref{prop: SubOpt LSVI under Uniform Coverage} is that the suboptimality gap contains no polynomial dependence on the ambient dimension $d$, making the result meaningful even in the high-dimensional regime $d > N$.  
Moreover, when $N \geq d$, the result of \citet{Zhang2021_RobustOfflineLinMDP} achieves $O(H^2\sqrt{d}\xi^{-1}\epsilon)$, since $\|\phi(\cdot)\|_2 \leq \sqrt{d}$ in our setting. The dependence on $\sqrt{d}$ makes this result less desirable compared to our result, $O(H^2 s \xi^{-1}\sqrt{\epsilon})$.

Moreover, our bound currently scales as $O(s \sqrt{\epsilon})$, which we believe may be further improved. 
In particular, if one could design an \texttt{SRLE} oracle satisfying $\|\widehat w - w_\star\|_1 = O( \sqrt{s}\epsilon)$, then a sharper $O(\sqrt{s}\epsilon)$ suboptimality rate would follow immediately. 
However, the best computationally-efficient oracle we are aware of, \texttt{SRLE1} (see Proposition \ref{prop: RSLS with data coverage}), achieves only $\|\widehat w - w_\star\|_1 \leq O(s\sqrt{\epsilon})$, which we suspect is suboptimal. 
Thus, further progress in sparse regression under corruption could directly tighten our results by replacing the \texttt{SRLE1} oracle.
\end{remark}

\subsection{Sparse Robust LSVI with Single Concentrability}\label{subsec: Sparse Robust LSVI with Single Concentrability}
Beyond the uniform coverage case, we now explain why pointwise pessimistic bonuses can lead to large error bounds. 
In this section, we assume Assumption \ref{assp: Relative condition number} holds. 
Since the uniform coverage assumption no longer applies, we must replace the \texttt{SRLE1} estimator with either \texttt{SRLE2} or \texttt{SRLE3}. 
For concreteness, we focus on the result using \texttt{SRLE2}; the same challenges and limitations of LSVI also apply when using \texttt{SRLE3}.
We begin by relating the suboptimality gap of pessimistic LSVI to an upper bound on the Bellman error.

\begin{theorem}[\citep{jin2020provably}]   \label{Theorem: Suboptimality for general MDP}
Choose a pessimistic bonus $(\Gamma_h)_{h=1}^H$ such that $(\underline Q_{h} - \mathbb B_h \underline Q_{h+1})(x,a) \leq \Gamma_h(x,a)$ for any $(x,a) \in \mathcal X \times \mathcal A$ and $h\in [H]$ .
Then, the LSVI algorithm (Algorithm \ref{alg: Sparse Robust LSVI}) outputs a policy $\widehat \pi$ such that 
$    \mathrm{SubOpt(\pi, \widehat \pi)} \leq 2 \sum_{h=1}^H \mathbb E_{\pi} \SquareBr{\Gamma_h(x,a)}.
$
\end{theorem}

Let $w_h^\star = \mathcal P_h^*(\underline Q_{h+1})$ denote the best-fit linear predictor given the pessimistic estimator at stage $(h+1)$.  
Let $\tilde S_h$ be the sparsity support of $\widehat w_h - w^\star_h$, so that $|\tilde S_h| \leq 2s$.  
The Bellman error and the pessimistic bonus $\Gamma_h$ can then be bounded as
\begin{equation} \label{eq: Bellman error in case of sparse MDP}
    \begin{aligned}
         &|\widehat Q_h(x,a) - \mathbb B_h(\underline Q_{h+1})(x,a)| 
         = |\phi(x,a)^\top(\widehat w_h - w^\star_h)| \\
         &\overset{(i)}{\leq} \underbrace{\norm{\widehat w_h - w^\star_h}_{\widehat \Sigma_{h}}}_{\text{Estimator error $\alpha_h$}} \;
         \norm{\phi_{\tilde S_h}(x,a)}_{\widehat\Sigma_{h}^{-1}} \\
         &\overset{(ii)}{\leq} \alpha_h
         \RoundBr{\max_{ S:\; |S|\leq 2s}\norm{\phi_{ S}(x,a)}_{\widehat\Sigma_{h}^{-1}}} \;\triangleq\; \Gamma_h(x,a). \\
    \end{aligned}
\end{equation}

Inequality $(i)$ holds because only features in $\tilde S_h$ of $\phi$ contribute non-zero terms to the Bellman error.  
The $\max$ operator in $(ii)$ is necessary since the true support $\tilde S_h$ is unknown.  
Thus, to guarantee pessimism, one must maximize over all subsets $S$ of size at most $2s$. 
We note that $\max_{S:\, |S|\leq 2s}\|\phi_{S}(x,a)\|_{\widehat\Sigma_{h}^{-1}}$ is not always  smaller than $\|\phi(x,a)\|_{\widehat\Sigma_{h}^{-1}}$. However, restricting $\phi(\cdot)$ to a sparse support is a natural way to remove the explicit dependence on $d$ in the Bellman error bound.
Nevertheless, this maximization introduces additional error, since
\begin{equation} \label{eq: error caused by exchanging max and expectation}
    \begin{aligned}
        &\underset{(x,a)\sim d^{\pi_\star}}{\mathbb E}[\Gamma_h(x,a)] \\
        &\asymp  \alpha_h\,\mathbb E_{d^{\pi_\star}}\! \SquareBr{\max_{S: |S| \leq 2s} \RoundBr{\phi^\top_{S}(x,a)\Sigma_{h}  \phi_{S}(x,a)}} \\
        &\overset{(i)}{\geq} \alpha_h \max_{S:|S|\leq 2s} \mathbb E_{d^{\pi_\star}}\SquareBr{\phi_S^\top(x,a)\Sigma_{h}  \phi_S(x,a)}. \\
    \end{aligned}
\end{equation}
where inequality $(i)$ follows from Jensen's inequality by exchanging expectation and maximization.  

The LHS of inequality~$(i)$ in \eqref{eq: error caused by exchanging max and expectation} serves as the ideal upper bound for the Bellman error, contributing a factor of order $\alpha_h \sqrt{s\kappa}$ under Assumption~\ref{assp: Relative condition number}~\citep{Zhang2021_RobustOfflineLinMDP}.
However, since sparsity support $\tilde S_h$ is unknown, the Bellman error must be bounded by the RHS, which can be significantly larger as shown below. 

\begin{restatable}{proposition}{LemmaMaxExpectation} \label{lem:  Max and Expectation swap}
    Let $\Sigma = \lambda I$. 
    Let $z \sim \mathrm{Ber}^d(\nicefrac{1}{2})$. Then, for $d > 4s$, we have
    \begin{equation}
    \begin{aligned}
        &\underbrace{\mathbb E_{z\sim\mu}\SquareBr{\max_{S:|S| = 2s} z_S^\top \Sigma^{-1} z_S}}_{\text{LHS in \eqref{eq: error caused by exchanging max and expectation}}} 
          - \underbrace{\max_{S:|S| = 2s} \mathbb E_{x\sim \mu} \SquareBr{z_S^\top \Sigma^{-1}z_S}}_{\text{RHS in \eqref{eq: error caused by exchanging max and expectation}}}\\& \geq (1-2\exp(-d/8))\frac{s}{\lambda}.
    \end{aligned}
    \end{equation}
\end{restatable}
Since $\lambda$ is choose typically of order $1/\sqrt{N}$, the resulting $s\log(d)\sqrt{N}$-multiplicative factor in the Bellman error, leading to a large suboptimality gap for LSVI.

\begin{remark}
The difficulty with LSVI is that it imposes pessimism in a pointwise manner.
This can be viewed as \emph{over-pessimism}. 
While such over-pessimism does not cause issues in low-dimensional MDPs, it leads to excessive Bellman error in high-dimensional sparse MDPs, where the sparsity support is hidden from the learner.
Let $\underline V_1(x_1) = \max_{a} \underline Q_1(x_1,a)$, from \citep{Jin2020_provableLSVI}, the main reason LSVI requires pointwise pessimistic bonus is the need to enforce 
\[
    \underline V_1(x_1) - V_1^{\hat\pi}(x_1) 
    = \mathbb{E}_{d^{\hat \pi}} \SquareBr{\sum_{h=1}^H (\underline Q_h - \Bell_h \underline Q_{h+1})(x,a)} < 0.
\]
Here, the policy $\hat \pi$ is defined greedily with respect to $(\underline Q_h)_{h=1}^H$, which means that $\hat \pi$ is chosen \emph{after} $(\underline Q_h)_{h=1}^H$ is fixed. 
This creates a statistical dependency: $\hat \pi$ tends to place weight on $(x,a)$ where $(\underline Q_h - \Bell_h \underline Q_{h+1})(x,a)$ is less negative (or even positive), due to the greedy nature of $\hat \pi$ (i.e., the $\max$ operator).
As a result, the only distribution-free way to guarantee the inequality above is to impose 
\begin{equation} \label{eq: pointwise pessimism}
    (\underline Q_h - \Bell_h \underline Q_{h+1})(x,a) \le 0 \quad \forall(x,a,h),
\end{equation}
i.e., using pointwise pessimistic bonus. 
We shall see that the AC method introduced in Section \ref{subsec: Sparse Robust AC without Uniform Coverage} avoids this requirement and can achieve a meaningful suboptimality gap as a result.
\end{remark}

\section{SPARSE ACTOR-CRITIC METHODS}
\label{section: Sparse Actor-Critic Methods}
The shortcoming of LSVI arises from its pointwise pessimistic bonus. In this section, we show that this issue can be addressed by incorporating sparsity directly into AC methods. 
The key idea is that, unlike LSVI, AC does not impose pessimism uniformly over all state–action pairs; instead, the critic only needs to evaluate the current actor’s policy pessimistically. 
This is better aligned with sparse structure, since the regression error is controlled only along the policy being optimized, thereby avoiding the overly conservative bonus that limits LSVI. 
As a result, sparsity, pessimism, and weak coverage can be combined in a natural way.

First, consider the actor. We use a log-linear class of policies.  
In particular, for any policy parameter $\upsilon = (\upsilon_h)_{h=1}^H$, we define
\[
\pi^{\upsilon}_{h}(a\mid x) = \frac{\exp(\AngleBr{\phi(x,a),\upsilon_h})}{\exp \RoundBr{\sum_{a \in \mathcal A}\AngleBr{\phi(x,a),\upsilon_h}}}.
\]
To update the actor, we invoke the mirror descent framework \citep{zanette2021_provablebenefitsactorcriticmethods}, that is,
$
\pi_{t+1,h} \propto  \pi_{t,h} \exp(\eta Q_{h,t}(x,a)).
$
When the $Q$-function is linear, that is, $Q_{h,t} \in \mathcal Q_h$, it is possible to simplify the actor update as shown in  
Line~\ref{line: actor update} of Algorithm~\ref{alg: actor-critic}.  
The pseudocode of Sparse Robust Actor-Critic is provided in Algorithm~\ref{alg: actor-critic}.

\begin{algorithm}[H]
    \caption{Sparse Robust Actor-Critic}\label{alg: actor-critic}

    \textbf{Actor:}
    \begin{algorithmic}[1]
        \STATE Input dataset $\mathcal D$, learning rate $\eta$, \texttt{SRLE} oracle. 
        \STATE Initialize $\upsilon_1 = \bm 0_d$. 
        \FOR{$t \in [T]$}
            \STATE $\underline w_t =$ \texttt{Critic}($\mathcal D$, $\pi_{\upsilon_t}$).
            \STATE $
                \upsilon_{t+1} = \upsilon_t + \eta \underline w_t \label{line: actor update}
            $
        \ENDFOR
    \end{algorithmic}

    \vspace{0.5em}\hrule\vspace{0.5em} 

    \textbf{Critic:}
    \begin{algorithmic}[1]
        \STATE Input dataset $\mathcal D$, \texttt{SRLE} oracle.
        \STATE Solve \texttt{PessOpt} subroutine, obtain  weight $\underline w$.
        \STATE Return pessimistic weight $\underline w$.
    \end{algorithmic}
\end{algorithm}
\subsection{Sparse Robust AC with Uniform Coverage} \label{subsec: Sparse Robust AC with Uniform Coverage}

We first consider the case where the uniform coverage assumption holds. 
Here, we can directly apply the \texttt{SRLE1} estimator to define the regression operator $\mathcal R^\pi$. 
This leads to a natural construction of the critic, which recursively estimates value functions in a pessimistic manner.  

Formally, the subroutine \texttt{PessOpt} within the critic is defined by the following pessimistic optimisation problem.
Let $\underline Q_{H+1} = \bm 0$. 
Given a policy $\pi$, for $h \in [H]$, we define recursively
\begin{optproblem}[uniform coverage]
\begin{equation}\label{eq: Critic with uniform coverage}
    \begin{aligned}
        \underline w_h \in &\argmin_{w_h} &&\norm{w_h -  \mathcal R_h^\pi(\underline Q_{h+1})}_1;\\
        &\text{s.t.} \quad  && \| w_h\|_1 \leq H+1-h.
    \end{aligned}
\end{equation}
\end{optproblem}

The overall algorithm alternates between actor and critic updates, summarized in Algorithms \ref{alg: actor-critic}.
The next result shows that this sparse robust actor-critic algorithm achieves strong suboptimality guarantees under uniform coverage.

\begin{restatable}{theorem}{TheoremSuboptimalGapACUC} \label{theorem: suboptimality gap of Actor-Critic with Uniform Coverage}
    Suppose Assumption \ref{assp: sparse linear mdp}  and \ref{assp: uniform coverage} hold.
    Then, with step size $\eta = \sqrt{\frac{\log \mathcal A}{N}}$, after $T=N/H$ iterations the actor-critic algorithm returns a policy $\hat \pi$ that satisfies
    \begin{equation}
    \begin{aligned}
        \mathrm{SubOpt}(\pi_\star, \hat \pi) = O \Biggl(&\frac{H^3s\log(dNH\delta^{-1})}{\xi \sqrt{N}} \\
        &+  H^3 \sqrt{\frac{\log (|\mathcal A|)}{N}} + \frac{H^3s\sqrt{\epsilon}}{\xi} \Biggl) , 
    \end{aligned}
    \end{equation}
    with probability at least $1-\delta$.
\end{restatable}

This result demonstrates that, under uniform coverage, actor-critic methods achieve near-optimal sample complexity in the sparse setting, similar to LSVI-based algorithms. 
Moreover, we will show that actor-critic methods are superior to LSVI when the data distribution is ill-conditioned.

\subsection{Sparse Robust AC with Single Concentrability} \label{subsec: Sparse Robust AC without Uniform Coverage}
We now turn to the more challenging setting where only the single-policy concentrability assumption holds. 
In this case, the critic must be modified to incorporate explicit pessimistic constraints. 
The actor-critic algorithm remains the same in structure, but the critic subproblem is more involved.

Formally, the subroutine \texttt{PessOpt} within the critic is defined by the following optimisation problem.
\begin{optproblem}[single concentrability]
\begin{equation} \label{eq: pessimistic Critic}
\begin{aligned}
   {\{\underline{ w}_h^\pi\}_{h=1}^H} \;\triangleq\; &\argmin_{\{ w_h\}_{h=1}^H} \; &&\sum_{a} \pi_1(a|x_1)\langle \phi(x_1,a),  w_1\rangle \\
    &\text{s.t. } \; && \| w_h - \mathcal R^\pi_h(\langle \phi ,  w_{h+1}\rangle) \|_{\hat \Sigma_h}^2 \leq \alpha_h^2,\\
    & && \| w_h \|_1 \leq H+1-h, \\
    & && \| w_h\|_0 \leq s.
\end{aligned}
\end{equation}
\end{optproblem}
\begin{remark}[How AC avoids pointwise pessimistic bonuses]
Notice that the objective function of the \texttt{PessOpt} subroutine only ensures that the value function is pessimistic at the initial state, that is, $\underline V_1^\pi(x_1) \leq V_1^\pi(x_1)$ for the given policy $\pi$.
It does not guarantee that the value function is pessimistic for any other $(x,a,h)$, which is the key difference compared to how LSVI imposes pointwise bonuses \eqref{eq: Bellman error in case of sparse MDP}.
We note that AC is able to bypass pointwise bonuses thanks to its algorithmic structure. 
In particular, at each phase $t$, the AC algorithm \emph{fixes a policy $\pi_t$ first}; then it is possible to compute a pessimistic $Q$ function that defines $\mathcal E_{h}^\pi(x,a)$, which satisfies
$
\underline V_1^{\pi_t}(x_1) - V_1^{\pi_t}(x_1)  < 0,
$
by solving the \texttt{PessOpt} subroutine \eqref{eq: pessimistic Critic}.
In contrast, in LSVI-type algorithms the policy $\hat \pi$ is defined greedily 
\emph{only after} $\underline Q_h(x,a)$ is constructed. 
Consequently, to safely guarantee $
    \underline V_1(x_1) - V^{\hat\pi}(x_1) < 0,
$
one must impose pessimism in the pointwise manner as in \eqref{eq: pointwise pessimism}.

\end{remark}

This simple observation allows us to derive the suboptimality-gap guarantee stated below.
\begin{restatable}{theorem}{ThmMainSubOptGapActorCritic} \label{theorem: suboptimality gap of Actor-Critic}
    Suppose Assumption \ref{assp: sparse linear mdp} and \ref{assp: Relative condition number} hold.
    Let operators $(\mathcal R^\pi_h)_{h=1}^H$ be defined using the \texttt{SRLE2} estimator.
    Choose the sequence $\alpha$ such that 
  \begin{align*}
      \alpha_h^2 =     O\Biggl(\frac{\sqrt{s}H^3\log(dH\delta^{-1})}{\sqrt{N}} + H^3\epsilon + H^3(\lambda +\epsilon)\Biggl).
  \end{align*}  
    Choose $\lambda = \frac{s}{N}\log\tfrac{d}{s\delta}.$
    Then, after $T = N$ iterations with step size $\eta = \sqrt{\frac{\log(|\mathcal A|)}{H^2N}}$, the actor-critic algorithm returns a policy $\hat \pi$ that satisfies
    \begin{equation*}
    \begin{aligned}
        \mathrm{SubOpt}(\pi_\star, \hat \pi) = O\Biggl( &\frac{H^3\sqrt{\kappa}s^{\nicefrac 3 4}\sqrt{\log(dHN\delta^{-1})}}{N^{\tfrac{1}{4}}}  \\
        &+  H^3 \sqrt{\frac{\log(|\mathcal A|)}{N}} + H^3\sqrt{\kappa s \epsilon} \Biggl),
    \end{aligned}
    \end{equation*}
    with probability at least $1-\delta$.
\end{restatable}
Theorem \ref{theorem: suboptimality gap of Actor-Critic} shows that, even under weaker coverage conditions, sparse actor-critic algorithms can still achieve meaningful robustness guarantees, though with slower rates compared to the uniform coverage case.
In particular, suppose $N = \Omega\RoundBr{\frac{s\log^2(dHN\delta^{-1})}{\epsilon^2}}$, 
then we obtain 
\begin{equation}
    \mathrm{SubOpt}(\pi_\star, \hat \pi) = \tilde O(H^3\sqrt{\kappa s \epsilon}).
\end{equation}

We note that the result in Theorem \ref{theorem: suboptimality gap of Actor-Critic} relies on the \texttt{SRLE2} oracle, which is computationally expensive. 
Therefore, we derive the next theorem, which leverages the more efficient \texttt{SRLE3} oracle, at the cost of a larger suboptimality gap.
\begin{restatable}{theorem}{ThmMainSubOptGapActorCriticMirrorDescent} \label{theorem: suboptimality gap of Actor-Critic with Mirror Descent}
    Suppose Assumption \ref{assp: sparse linear mdp} and \ref{assp: Relative condition number} hold.
    Let operator $\mathcal R^\pi_h$ be defined using the \texttt{SRLE3} estimator.
    Choose the sequence $\alpha$ such that 
    \begin{align*}
        \alpha_h = O\RoundBr{\frac{H^3\log(dH\delta^{-1})}{\sqrt{N}} + H^3\sqrt{\epsilon} + H^3(\lambda +\epsilon)}.
    \end{align*}
    Choose $ \lambda = \frac{s}{N}\log\tfrac{d}{s\delta}.$
    After $T =  N$ iterations with step size $\eta = \sqrt{\frac{\log(|\mathcal A|)}{H^2 N}}$, the actor-critic algorithm returns a policy $\hat \pi$ that satisfies
    \begin{equation}
    \begin{aligned}
        \mathrm{SubOpt}(\pi_\star, \hat \pi) = O\Biggl( 
        &\frac{H^3\sqrt{\kappa s}\sqrt{\log(dHN\delta^{-1})}}{N^{\tfrac{1}{4}}}  \\
        &+  H^3 \sqrt{\frac{\log(|\mathcal A|)}{N}} + H^3\sqrt{\kappa s}\epsilon^{\frac 1 4} \Biggl),
    \end{aligned}
    \end{equation}
    with probability at least $1-\delta$.
\end{restatable}
Theorem \ref{theorem: suboptimality gap of Actor-Critic with Mirror Descent} shows that using computationally efficient oracle \texttt{SRLE3}, sparse actor-critic algorithms still achieve meaningful robustness guarantees, although slightly worse compared to that of Theorem \ref{theorem: suboptimality gap of Actor-Critic}.
In particular, suppose $N = \Omega\RoundBr{\frac{\log^2(dHN\delta^{-1})}{\epsilon}}$, 
then we obtain 
\begin{equation}
    \mathrm{SubOpt}(\pi_\star, \hat \pi) = \tilde O(H^3\sqrt{\kappa s }\epsilon^{\frac{1}{4}}).
\end{equation}

\begin{proof}[Proof sketch of Theorem \ref{theorem: suboptimality gap of Actor-Critic} and \ref{theorem: suboptimality gap of Actor-Critic with Mirror Descent}]
The proof has three main steps. First, for each actor iterate \(\pi_t\), we solve the constrained critic problem in (14) and use its solution to define an induced MDP \(M_t\), similar to that in \cite{zanette2021_provablebenefitsactorcriticmethods}. By construction, the critic is exact in \(M_t\), and the induced reward perturbation is chosen so that \(V^{\pi_t}_{M_t}\) is no larger than the true value \(V^{\pi_{t}}\). Thus, the critic provides a pessimistic evaluation of the current policy without requiring a pointwise lower bound over all state-action pairs.

Second, the sparse robust regression oracle controls the critic error in the empirical covariance norm. The key step is then to transfer this control from the data distribution to the occupancy measure of the comparator policy. Under single-policy concentrability, this transfer incurs only a \(\sqrt{\kappa}\) factor, while sparsity ensures that the complexity depends on the support size \(s\), rather than the ambient dimension \(d\). This yields a bound on the suboptimality of evaluating \(\pi_\star\) in the induced MDP.

Third, we combine this critic bound with the standard mirror-descent regret guarantee for the actor updates. Averaging over iterations and summing the two contributions gives the final suboptimality bound. In particular, the result follows from the interaction of three ingredients: pessimistic policy evaluation through the induced MDP, sparse regression error control under weak coverage, and no-regret policy optimization.
\end{proof}
In the uniform coverage case discussed in Sections \ref{subsec: Sparse Robust AC with Uniform Coverage} and \ref{subsec: Sparse Robust LSVI with Uniform Coverage}, both Algorithm~\ref{alg: Sparse Robust LSVI} and Algorithm~\ref{alg: actor-critic} can be implemented in $\mathrm{poly}(N,d,H,s)$ time. 
In contrast, in the single concentrability setting, a necessary condition for computational efficiency is to use the \texttt{SRLE3} estimator in place of \texttt{SRLE2}. 
However, even with \texttt{SRLE3}, the Actor–Critic method still requires solving Equation~\eqref{eq: pessimistic Critic}, which involves an $\ell_0$-type constraint. 
This constraint substantially increases the computational complexity, making Equation~\eqref{eq: pessimistic Critic} the primary bottleneck in achieving a polynomial-time algorithm. 

We note that dropping the $\ell_0$-constraint in \eqref{eq: pessimistic Critic}  may cause significant errors.
Without restricting $\|\phi\|_{\Sigma_h^{-1}}$ to a sparsity support, the suboptimality gap incurs a multiplicative factor $\sqrt{\kappa d}$, since
$
\mathbb{E}_{d^{\pi_\star}}\!\left[\|\phi\|_{\Sigma_h^{-1}}\right] = O(\sqrt{\kappa d})$
in the worst case \citep{Zhang2021_RobustOfflineLinMDP}.  
Moreover, without additional assumptions on the data distribution $\nu$, it remains unclear how to design an algorithm that solves Equation~\eqref{eq: pessimistic Critic} in $\mathrm{poly}(N,d,s)$ time. 
We conjecture that some relaxation of data distribution $\nu$ may be needed to make such a solution computationally feasible.

\section{CONCLUSION}
We studied offline RL in sparse linear MDPs with limited coverage and adversarial corruption, showing that pessimistic LSVI can fail in sparse regimes due to overly conservative pointwise pessimistic bonuses. To address this, we proposed a pessimistic actor–critic framework with sparse, robust regression oracles, achieving near-optimal guarantees when $d \gg N$ under single-policy concentrability. Our results provide the first non-vacuous bounds for sparse offline RL, establishing a separation between LSVI and AC methods.

An important direction for future work is to explore algorithmic relaxations of the $\ell_0$-constraint in Equation~\eqref{eq: pessimistic Critic} that preserve statistical guarantees while reducing computational burden. In particular, investigating distributional assumptions on $\nu$ or alternative convex surrogates could pave the way toward polynomial-time solutions.


It would also be interesting to study multi-agent Markov games under sparsity. The structure of sparse linear MDPs can be naturally extended to multi-agent Markov games. However, existing algorithms for solving offline Markov games are based on value iteration based methods, and handling sparsity will require generalization of our actor-critic framework to multi-agent setting. Furthermore, recent work~\cite{nika2024corruption} on corruption robustness in offline Markov games has highlighted that data coverage assumptions play an important role on achievable error rates, and it would be interesting to explore how sparsity influences such data coverage assumptions.

\subsubsection*{Acknowledgements}
The work of Andi Nika and Goran Radanovic was
funded by the Deutsche Forschungsgemeinschaft (DFG,
German Research Foundation) – project number 467367360.

\bibliographystyle{plainnat}
\bibliography{Ref}

@inproceedings{Kim2023_SparseRL,
  title     = {A Doubly Robust Approach to Sparse Reinforcement Learning},
  author    = {Wonyoung Kim and Garud Iyengar and Assaf Zeevi},
  booktitle = {Proceedings of the 27th International Conference on Artificial Intelligence and Statistics},
  year      = {2024},
}

@incollection{lange2012batch,
  title={Batch reinforcement learning},
  author={Lange, Sascha and Gabel, Thomas and Riedmiller, Martin},
  booktitle={Reinforcement learning: State-of-the-art},
  pages={45--73},
  year={2012},
  publisher={Springer}
}

@article{levine2020offline,
  title={Offline reinforcement learning: Tutorial, review, and perspectives on open problems},
  author={Levine, Sergey and Kumar, Aviral and Tucker, George and Fu, Justin},
  journal={arXiv preprint arXiv:2005.01643},
  year={2020}
}

@inproceedings{kidambi2020morel,
  title     = {MOReL: Model-Based Offline Reinforcement Learning},
  author    = {Rahul Kidambi and Aravind Rajeswaran and Praneeth Netrapalli and Thorsten Joachims},
  booktitle = {Advances in Neural Information Processing Systems 33},
  year      = {2020},
}

@inproceedings{bakshi2021robust,
  title={Robust linear regression: Optimal rates in polynomial time},
  author={Bakshi, Ainesh and Prasad, Adarsh},
  booktitle={Proceedings of the 53rd Annual ACM SIGACT Symposium on Theory of Computing},
  pages={102--115},
  year={2021}
}

@inproceedings{kumar2020conservative,
  title={Conservative {Q}-learning for {O}ffline {R}einforcement {L}earning},
  author={Kumar, Aviral and Zhou, Aurick and Tucker, George and Levine, Sergey},
  booktitle={Advances in Neural Information Processing Systems},
  year={2020}
}

@inproceedings{klivans2018efficient,
  title={Efficient algorithms for outlier-robust regression},
  author={Klivans, Adam and Kothari, Pravesh K and Meka, Raghu},
  booktitle={Conference On Learning Theory},
  year={2018},
}

@article{yu2020mopo,
  title={Mopo: Model-based offline policy optimization},
  author={Yu, Tianhe and Thomas, Garrett and Yu, Lantao and Ermon, Stefano and Zou, James Y and Levine, Sergey and Finn, Chelsea and Ma, Tengyu},
  journal={Advances in Neural Information Processing Systems},
  year={2020}
}

@inproceedings{ye2023corruption,
  title={Corruption-robust offline reinforcement learning with general function approximation},
  author={Ye, Chenlu and Yang, Rui and Gu, Quanquan and Zhang, Tong},
  booktitle={Advances in Neural Information Processing Systems},
  year={2023}
}

@article{liu2025multi,
  title={Multi-level Certified Defense Against Poisoning Attacks in Offline Reinforcement Learning},
  author={Liu, Shijie and Cullen, Andrew C and Montague, Paul and Erfani, Sarah and Rubinstein, Benjamin IP},
  journal={arXiv preprint arXiv:2505.20621},
  year={2025}
}

@article{rakhsha2021policy,
  title={Policy {T}eaching in {R}einforcement {L}earning via {E}nvironment {P}oisoning {A}ttacks},
  author={Rakhsha, Amin and Radanovic, Goran and Devidze, Rati and Zhu, Xiaojin and Singla, Adish},
  journal={Journal of Machine Learning Research},
  year={2021}
}

@inproceedings{sun2020stealthy,
  title={Stealthy and {E}fficient {A}dversarial {A}ttacks {A}gainst {D}eep {R}einforcement {L}earning},
  author={Sun, Jianwen and Zhang, Tianwei and Xie, Xiaofei and Ma, Lei and Zheng, Yan and Chen, Kangjie and Liu, Yang},
  booktitle={Association for the Advancement of Artificial Intelligence},
  year={2020}
}

@inproceedings{lin2017tactics,
  title={Tactics of {A}dversarial {A}ttack on {D}eep {R}einforcement {L}earning {A}gents},
  author={Lin, Yen-Chen and Hong, Zhang-Wei and Liao, Yuan-Hong and Shih, Meng-Li and Liu, Ming-Yu and Sun, Min},
  booktitle={International Joint Conference on Artificial Intelligence},
  Xpages={3756--3762},
  year={2017}
}

@article{huang2017adversarial,
  title={{A}dversarial {A}ttacks on {N}eural {N}etwork {P}olicies},
  author={Huang, Sandy and Papernot, Nicolas and Goodfellow, Ian and Duan, Yan and Abbeel, Pieter},
  Xjournal={arXiv preprint arXiv:1702.02284},
  journal   = {CoRR},
  volume    = {abs/1702.02284},
  year={2017}
}

@article{wu2019behavior,
  title={Behavior regularized offline reinforcement learning},
  author={Wu, Yifan and Tucker, George and Nachum, Ofir},
  journal={arXiv preprint arXiv:1911.11361},
  year={2019}
}

@article{jaques2019way,
  title={{W}ay {O}ff-policy {B}atch {D}eep {R}einforcement {L}earning of {I}mplicit {H}uman {P}references in {D}ialog},
  author={Jaques, Natasha and others},
  journal={CoRR},
  volume={abs/1907.00456},
  year={2019}
}

@article{kumar2019stabilizing,
  title={Stabilizing off-policy q-learning via bootstrapping error reduction},
  author={Kumar, Aviral and Fu, Justin and Soh, Matthew and Tucker, George and Levine, Sergey},
  journal={Advances in neural information processing systems},
  volume={32},
  year={2019}
}

@inproceedings{
buckman2020importance,
title={The Importance of Pessimism in Fixed-Dataset Policy Optimization},
author={Jacob Buckman and Carles Gelada and Marc G Bellemare},
booktitle={International Conference on Learning Representations},
year={2021},
}

@article{liu2020provably,
  title={Provably good batch off-policy reinforcement learning without great exploration},
  author={Liu, Yao and Swaminathan, Adith and Agarwal, Alekh and Brunskill, Emma},
  journal={Advances in neural information processing systems},
  year={2020}
}

@inproceedings{xie2021bellman,
  title={Bellman-consistent {P}essimism for {O}ffline {R}einforcement {L}earning},
  author={Xie, Tengyang and Cheng, Ching-An and Jiang, Nan and Mineiro, Paul and Agarwal, Alekh},
  booktitle={Advances in neural information processing systems},
  year={2021}
}

@inproceedings{laroche2019safe,
  title={Safe {P}olicy {I}mprovement with {B}aseline {B}ootstrapping},
  author={Laroche, Romain and Trichelair, Paul and Des Combes, Remi Tachet},
  booktitle={International Conference on Machine Learning},
  year={2019},
}

@inproceedings{Golowich2023_SparseRL,
author = {Golowich, Noah and Moitra, Ankur and Rohatgi, Dhruv},
title = {Exploring and Learning in Sparse Linear MDPs without Computationally Intractable Oracles},
year = {2024},
booktitle = {Proceedings of the 56th Annual ACM Symposium on Theory of Computing},
}

@inproceedings{jin2020provably,
  title={Provably efficient reinforcement learning with linear function approximation},
  author={Jin, Chi and Yang, Zhuoran and Wang, Zhaoran and Jordan, Michael I},
  booktitle={Conference on learning theory},
  year={2020},
}

@inproceedings{Zhang2021_RobustOfflineLinMDP,
   author = {Xuezhou Zhang and Yiding Chen and Jerry Zhu and Wen Sun},
   title = {Corruption-Robust Offline Reinforcement Learning},
   booktitle = {International Conference on Artificial Intelligence and Statistics},
   year = {2021},
}

@inproceedings{Hao2020_OfflineSparseRL,
  author={Botao Hao and Yaqi Duan and Tor Lattimore and Csaba Szepesvári and Mengdi Wang},
  title={Sparse Feature Selection Makes Batch Reinforcement Learning More Sample Efficient},
  year={2021},
  booktitle={International Conference on Machine Learning},
}

@article{Merad2022_RobustSparseLinearRegressionHeavyTailDist,
   author = {Ibrahim Merad and Stéphane Gaïffas},
   title = {Robust Methods for High-Dimensional Linear Learning},
  journal      = {Journal of Machine Learning Research},
   year = {2022},
}

@article{Hsu2014_RLS_GeneralRandomDesign,
   author = {Daniel Hsu and Sham M. Kakade and Tong Zhang},
   journal = {Foundations of Computational Mathematics},
   title = {Random Design Analysis of Ridge Regression},
   volume = {14},
   year = {2014}
}

@inproceedings{zanette2021_provablebenefitsactorcriticmethods,
title={Provable Benefits of Actor-Critic Methods for Offline Reinforcement Learning},
author={Andrea Zanette and Martin Wainwright and Emma Brunskill},
booktitle={Advances in Neural Information Processing Systems},
year={2021},
}

@inproceedings{Jin2020_provableLSVI,
  author    = {Ying Jin and Zhuoran Yang and Zhaoran Wang},
  title     = {Is Pessimism Provably Efficient for Offline RL?},
  booktitle = {Proceedings of the 38th International Conference on Machine Learning},
  year      = {2021},
}

@InProceedings{zanette2021cautiouslyoptimisticpolicyoptimization,
  title = 	 {Cautiously Optimistic Policy Optimization and Exploration with Linear Function Approximation},
  author =       {Zanette, Andrea and Cheng, Ching-An and Agarwal, Alekh},
  booktitle = 	 {Proceedings of Thirty Fourth Conference on Learning Theory},
  year = 	 {2021},
}

@inproceedings{mandal2025corruption,
  title={Corruption Robust Offline Reinforcement Learning with Human Feedback},
  author={Mandal, Debmalya and Nika, Andi and Kamalaruban, Parameswaran and Singla, Adish and Radanovic, Goran},
  booktitle={International Conference on Artificial Intelligence and Statistics},
  pages={4429--4437},
  year={2025},
  organization={PMLR}
}

@inproceedings{nika2024corruption,
  title={Corruption-robust offline two-player zero-sum markov games},
  author={Nika, Andi and Mandal, Debmalya and Singla, Adish and Radanovic, Goran},
  booktitle={International Conference on Artificial Intelligence and Statistics},
  pages={1243--1251},
  year={2024},
  organization={PMLR}
}

\section*{Checklist}

\begin{enumerate}

  \item For all models and algorithms presented, check if you include:
  \begin{enumerate}
    \item A clear description of the mathematical setting, assumptions, algorithm, and/or model. [Yes]
    \item An analysis of the properties and complexity (time, space, sample size) of any algorithm. [Yes]
    \item (Optional) Anonymized source code, with specification of all dependencies, including external libraries. [Not Applicable]
  \end{enumerate}

  \item For any theoretical claim, check if you include:
  \begin{enumerate}
    \item Statements of the full set of assumptions of all theoretical results. [Yes]
    \item Complete proofs of all theoretical results. [Yes]
    \item Clear explanations of any assumptions. [Yes]     
  \end{enumerate}

  \item For all figures and tables that present empirical results, check if you include:
  \begin{enumerate}
    \item The code, data, and instructions needed to reproduce the main experimental results (either in the supplemental material or as a URL). [Not Applicable]
    \item All the training details (e.g., data splits, hyperparameters, how they were chosen). [Not Applicable]
    \item A clear definition of the specific measure or statistics and error bars (e.g., with respect to the random seed after running experiments multiple times). [Not Applicable]
    \item A description of the computing infrastructure used. (e.g., type of GPUs, internal cluster, or cloud provider). [Not Applicable]
  \end{enumerate}

  \item If you are using existing assets (e.g., code, data, models) or curating/releasing new assets, check if you include:
  \begin{enumerate}
    \item Citations of the creator If your work uses existing assets. [Not Applicable]
    \item The license information of the assets, if applicable. [Not Applicable]
    \item New assets either in the supplemental material or as a URL, if applicable. [Not Applicable]
    \item Information about consent from data providers/curators. [Not Applicable]
    \item Discussion of sensible content if applicable, e.g., personally identifiable information or offensive content. [Not Applicable]
  \end{enumerate}

  \item If you used crowdsourcing or conducted research with human subjects, check if you include:
  \begin{enumerate}
    \item The full text of instructions given to participants and screenshots. [Not Applicable]
    \item Descriptions of potential participant risks, with links to Institutional Review Board (IRB) approvals if applicable. [Not Applicable]
    \item The estimated hourly wage paid to participants and the total amount spent on participant compensation. [Not Applicable]
  \end{enumerate}

\end{enumerate}

\clearpage
\appendix

\thispagestyle{empty}
\renewcommand{\thetable}{\Alph{table}}
\setcounter{table}{0}

\onecolumn
\aistatstitle{
Sparse Offline Reinforcement Learning with Corruption Robustness}
\setcounter{page}{1}
\renewcommand\thepage{\roman{page}}

\renewcommand*\contentsname{Contents of Appendix}
\addtocontents{toc}{\protect\setcounter{tocdepth}{2}}
\doublespacing

\tableofcontents
\singlespacing

\clearpage
\appendix

\thispagestyle{empty}
\renewcommand{\thetable}{\Alph{table}}
\setcounter{table}{0}

\onecolumn
\aistatstitle{
Sparse Offline Reinforcement Learning with Corruption Robustness}
\setcounter{page}{1}
\renewcommand\thepage{\roman{page}}

\renewcommand*\contentsname{Contents of Appendix}
\addtocontents{toc}{\protect\setcounter{tocdepth}{2}}
\doublespacing

\singlespacing

\section{Related Works} \label{appendix: related works}

\paragraph{Offline RL.} In recent years, offline RL \citep{lange2012batch, levine2020offline} has garnered considerable attention, as it offers a compelling alternative to the constraints of online data collection, making it particularly attractive for critical applications. Substantial progress has been made in this area, both on the empirical \citep{laroche2019safe, jaques2019way, kumar2020conservative, kidambi2020morel} and theoretical fronts \citep{xie2021bellman, jin2020provably}. The central challenge in offline RL is learning an optimal policy from data that provides only partial information about the environment. To address this challenge, value-iteration-based methods \citep{buckman2020importance, liu2020provably, kumar2019stabilizing, jin2020provably, yu2020mopo} and actor-critic methods \citep{levine2020offline, wu2019behavior, zanette2021_provablebenefitsactorcriticmethods} are among the most popular approaches. We consider methods that apply pessimism by penalizing value functions of policies that are under-represented in the data. However, as we show in Section \ref{subsec: Sparse Robust LSVI with Single Concentrability}, such an approach may lead to vacuous bounds in the sparse setting, arguably due to an overcompensation for uncertainty arising from the use of pointwise pessimistic bonuses. This also motivates the use of actor-critic methods, which we show to effectively mitigate such an issue. 

\paragraph{Corruption-robust offline RL.} There exists a large body of literature on adversarial attacks in RL \citep{sun2020stealthy, lin2017tactics, huang2017adversarial, rakhsha2021policy}, including training-time attacks, test-time attacks, and backdoor attacks. In this paper, we consider a special type of training-time attacks, namely data poisoning attacks in offline RL \citep{Zhang2021_RobustOfflineLinMDP}. To defend against such attacks, several approaches have been proposed, such as robust value function estimation based on robust statistical approaches \citep{Zhang2021_RobustOfflineLinMDP}, uncertainty weighting-based methods \citep{ye2023corruption} and, more recently, differential privacy-based methods \citep{liu2025multi}.
In this work, we employ a value function estimation approach \citep{Zhang2021_RobustOfflineLinMDP} based on sparse robust estimation \citep{Merad2022_RobustSparseLinearRegressionHeavyTailDist}. Under uniform coverage, the use of robust linear regression oracles \citep{bakshi2021robust, klivans2018efficient} in offline RL is straightforward.
However, under weaker notions of coverage in sparse settings, we show that integrating sparse oracles into LSVI-type algorithms, a standard approach for robust offline RL, is much more challenging and may yield suboptimal bounds that depend on the full dimension, whereas dependence only on the sparse dimension would be desirable.
In contrast, we demonstrate that sparse robust estimation can be naturally integrated into the actor–critic framework, effectively removing full-dimensional dependence and achieving bounds that scale only with the sparse dimension.

\paragraph{Sparse Linear MDP.} 
Sparse linear MDPs have been a primary focus in the online RL literature~\citep{Golowich2023_SparseRL,Kim2023_SparseRL,Hao2020_OfflineSparseRL}, where one can often design exploratory policies to ensure good data coverage and thereby exploit the underlying sparsity structure. In contrast, there are far fewer works on sparse MDPs in the offline RL setting \citep{Hao2020_OfflineSparseRL}, where strong data coverage assumptions (e.g., uniform coverage) are typically imposed to bypass the need for pessimism, as we shall explain later in this paper. However, we note that \emph{limited data coverage} is the central challenge of offline RL compared to online RL, without an exploratory policy that guarantees sufficient coverage, it is unclear from prior work whether one can learn a near-optimal policy. For this reason, our paper primarily focuses on the case of limited coverage, namely \emph{single-policy concentrability}, and is the first to establish non-vacuous suboptimality guarantees for sparse MDPs in this regime.

\section{Properties of Sparse Linear MDPs}
\begin{claim} \label{claim: norm of true vector}
    For any $Q_{h+1} \in \mathcal Q_{h+1}$ and policy $\pi$, let $ w_h^\pi \triangleq \mathcal P_h^\pi(Q_{h+1})$, then
    \[
     \|w_h^\pi \|_0 \leq s, \;\quad \|w_h^\pi \|_1 \leq (H-h+1).
    \]
\end{claim}
\begin{proof}
    By construction, we have that,
    \begin{equation}
\begin{aligned}
    \| w_h^\pi \|_1 &= \left\|\theta_h + \int_{\mathcal X} \RoundBr{\sum_{a'\in \mathcal A}\pi(a'\mid x') Q_{h+1}(x',a')} \mu_h(x') dx' \right\|_1  \\
    &\leq \|\theta_h\|_1 + \|(H-h) \mu_h(\mathcal X)  \|_1 \\
    &\leq H-h+1.
\end{aligned}
\end{equation}
Moreover, both $\theta_h$ and $\mu_h(x)$ are sparse vector supported in $S$, therefore, $w_h^\pi$ has support in $S$, and $\|w_h^\pi\|_0 \leq s$. 
\end{proof}

\begin{claim} \label{claim: bound of output}
    For any $Q_{h+1} \in \mathcal Q_{h+1}$ and policy $\pi$, for any $(x',a') \in \mathcal X \times \mathcal A$, we have that
    \begin{equation}
        \begin{aligned}
                \left|R_h + \sum_{a'\in \mathcal A}\pi(a'\mid x') Q_{h+1}(x',a')\right| & \leq H-h+1\leq  2(H-h)\\
                \left|R_h + \max_{a'\in \mathcal A} Q_{h+1}(x',a')\right| & \leq H-h+1\leq  2(H-h).
        \end{aligned}
    \end{equation}
\end{claim}
\begin{proof}
    The proof is immediately followed from the fact that $|R_h| \leq 1$ and $|Q_{h+1}(x',a')| \leq (H-h)$ for any pair $(x',a')$. 
\end{proof}
\section{Proof of Section \ref{Section: Sparse Robust Linear Regression}} \label{appendix: SRLE}
\subsection{Estimation Error of \texttt{SRLE1}} \label{appendix: SRLE1}
\texttt{SRLE1} was proposed in \citep{Merad2022_RobustSparseLinearRegressionHeavyTailDist}. 
It solves the convex optimization problem
\begin{equation}
\begin{aligned}
        \min_{w \in \mathbb{R}^d}\; & \|Xw - Y\|_2^2  \\
        \text{s.t.}\; & \|w \|_1 \leq \|w_\star \|_1,
\end{aligned}
\end{equation}
and achieves fast rates by exploiting the uniform coverage assumption via a multi-stage mirror descent algorithm. 
For further details, we refer the reader to the original paper \citep{Merad2022_RobustSparseLinearRegressionHeavyTailDist}. 
Below, we state the error guarantee of this estimator for our setting.

\PropRSLSwithDataCoverage*
\begin{proof}
    The proof is a result of Corollary 9 in \citep{Merad2022_RobustSparseLinearRegressionHeavyTailDist}, which can be stated as follows.
    Let 
    \[
    \sigma_{\max}^2 \triangleq \max_{ w: \|  w\|_1\leq \| w_\star \|_1} \max_{j\in [d]} \mathbb E_{z,y}\SquareBr{((z^\top  w -y)x_j)^2}.
    \]
    Then, the excessive risk is bounded as 
    \begin{equation}
    \norm{\hat  w - w_\star}_1 \leq O\RoundBr{\frac{ \sigma_{\max} s\log(d\delta^{-1})}{\xi \sqrt{N}} + \frac{\sigma_{\max}  s\sqrt{\epsilon}}{\xi}}.
    \end{equation}
    We now compute the constant $\sigma_{\max}$.
    Note that, since $\|z\|_\infty \leq 1$, then
    \[
    \sigma_{\max}^2 \leq \max_{ w}\mathbb E_{z,y}\SquareBr{(z^\top  w -y)^2} \leq \max_{ w}\max_{z} \mathbb E_{z,y}\SquareBr{(z^\top ( w - w_\star))^2} + \mathbb E[\eta^2] = 4\| w_\star\|_1^2 + \sigma^2.
    \]
    Plug this upper bound for $\sigma_{\max}$ into the above excessive risk, we conclude the proof.
\end{proof}

\subsection{Estimation Error of \texttt{SRLE2}}\label{appendix: SRLE2}
In this section, we will present the proof of Proposition \ref{theorem: RSLS without data coverage}.
We first introduce the \texttt{SRLE2} estimator. 
For any $C \subset [N]$, let $\hat \Sigma(C) \triangleq  \frac{1}{|C|} \sum_{i\in C} x_i x_i^\top$. 
Consider $\ell_0-\ell_2$ minimisation with trimmed mean square.

\begin{equation} \label{eq: l2-l0 regression}
\begin{aligned}
         &\min_{C\subset [N]}\min_ w && \frac{1}{N}\|Y_C- Z_C w \|_2^2 + \lambda \|  w\|_2^2 \\
        &\text{s.t. } && \| w \|_0 \leq s. \\
        & && \| w \|_1 \leq \|w_\star \|_1. \\
        & && |C| = (1-\epsilon)n
\end{aligned}
\end{equation}
Where $Z_C = [z_i]_{i\in C}$ and $Y_C = [y_i]_{i\in C}$. 
We first have the upper bound for prediction error.
\begin{lemma} \label{prop: Prediction error SRLE}\
    Let $C_*$ be the index set of clean data and $\hat C$ be the index set returned by optimisation problem~\eqref{eq: l2-l0 regression}.
    For all $S \in [d]$ such that $|S| = s$ , we have that
    \begin{equation}
        \|\widehat  w -  w_\star \|_{\widehat \Sigma(\hat C\cap C_*) + \lambda I}^2 = O\RoundBr{\frac{\sigma^2 s}{N\lambda} + \frac{\sigma^2 s \log(d\delta^{-1})}{N} + 2\lambda \| w_\star \|_2^2 + \sigma^2\epsilon}
    \end{equation}
\end{lemma}
\begin{proof}
    Based on the optimisation criteria, one has that
    \begin{equation}
        \frac{1}{N}\norm{Y_{\widehat C} - Z_{\widehat C}\widehat  w}_2^2 +\lambda\|\widehat  w \|_2^2 \leq \frac{1}{N} \norm{Y_{C_*} - Z_{C_*}  w_\star}_2^2 + \lambda\| w_\star \|_2^2
    \end{equation}
This leads to
\begin{equation*}
    \begin{aligned}
            &\frac{1}{N}\norm{Y_{\widehat C } - Z_{\widehat C}\widehat  w}_2^2 +\lambda\|\widehat  w \|_2^2 &&\leq \frac{1}{N} \norm{Y_{C_*} - Z_{C_*}  w_\star}_2^2 + \lambda\| w_\star \|_2^2 \\
\Longrightarrow\; & \frac{1}{N}\norm{Y_{\widehat C \cap C_*} - Z_{\widehat C\cap C_*}\widehat  w}_2^2 +\lambda\|\widehat  w \|_2^2  && \leq \frac{1}{N} \norm{Y_{C_*} - Z_{C_*}  w_\star}_2^2 + \lambda\| w_\star \|_2^2 \\
\iff\;        & \frac{1}{N}\norm{Y_{\widehat C \cap C_*}- Z_{\widehat C\cap C_*}\widehat  w}_2^2 + \lambda\|\widehat  w \|_2^2 + \lambda\|  w_\star \|_2^2 &&\leq\frac{1}{N} \norm{Y_{C_* \cap \widehat C} - Z_{C_* \cap \widehat C}  w_\star}_2^2  \\
& &&+ \frac{1}{N} \norm{Y_{C_* \setminus \widehat C} - Z_{C_* \setminus \widehat C}  w_\star}_2^2 + 2\lambda\| w_\star \|_2^2 \\
\Longrightarrow\;  & \frac{1}{N} \norm{Y_{C_* \cap \widehat C} - Z_{C_* \cap C} \widehat  w}_2^2   + \lambda\|\widehat  w -  w_\star \|_2^2 &&\leq \frac{1}{N} \norm{Y_{C_* \setminus \widehat C} - Z_{C_* \setminus \widehat C}  w_\star}_2^2 \\
& &&+  \frac{1}{N} \norm{Y_{C_* \cap \widehat C} - Z_{C_* \cap \widehat C}  w_\star}_2^2 + 2\lambda\| w_\star \|_2^2.\\
    \end{aligned}     
\end{equation*}

Since $Y_{\widehat C\cap C_\star} = Z_{\widehat C\cap C_\star} w_\star + \eta_{\widehat C\cap C_\star}$, we have that $\| Y_{\widehat C\cap C_\star} - Z_{\widehat C\cap C_\star}\widehat  w\|_2^2 = \|Z_{\widehat C\cap C_\star}( w_\star -\widehat  w) + \eta_{\widehat C\cap C_\star}\|_2^2$. Therefore, the inequality above can be written as
\begin{equation}
    \begin{aligned}
\frac{1}{N} \norm{Z_{C_* \cap \widehat C}(\widehat  w -  w_\star)}_2^2 + \lambda \norm{\widehat  w -  w_\star}_2^2 \leq  \underbrace{\frac{2}{N}\AngleBr{\eta_{C_* \cap \widehat C}, Z_{C_* \cap \widehat C}(\widehat  w -  w_\star)}}_{(1)} 
+ \underbrace{\frac{1}{N} \norm{Y_{C_* \setminus \widehat C} - Z_{C_* \setminus \widehat C}  w_\star}_2^2}_{(2)} + 2\lambda\| w_\star \|_2^2\\
    \end{aligned}
\end{equation}

Consider (1), and  note that $C_* \cap \widehat C$ is a uncorrupted data set and $( w_\star - \widehat  w) \in \tilde  S$ such that $\mathrm{dim}(\tilde S) \leq 2s$. 
\begin{claim}
    With probability  $1-\delta$, we have that
    \begin{equation}
        2\AngleBr{Z_{C_* \cap \widehat C}(\widehat  w -  w_\star), \eta_{C \cap C_*}}  \leq \frac{1}{2}\RoundBr{\|Z_{ C_* \cap C}( w_\star -\widehat  w) \|_2^2 + N\lambda \|\widehat  w -  w_\star \|_2^2}  + O\RoundBr{\frac{\sigma^2 s}{\lambda} + \sigma^2 s \log(d\delta^{-1})}.
    \end{equation}
\end{claim}
\begin{proof}
    We drop the subscript for data set $C$ and subspace $S$ to reduce clutter. Note that, $ w_\star -\hat  w$ have support $S$ such that $|S|\leq 2s$.
\begin{equation}
    \begin{aligned}
        2\AngleBr{(\widehat  w -  w_\star), Z_S^\top\eta} 
        &= 2\AngleBr{(Z_S^\top Z_S + N\lambda I)^{\nicefrac 1 2}(\widehat  w -  w_\star),  (Z_S^\top Z_S + N\lambda I)^{\nicefrac{-1}{2}} Z_S^\top\eta} \\
        &\leq \frac{1}{2}\RoundBr{\|Z_S( w_\star -\widehat  w) \|_2^2 + N\lambda \|( w_\star -\widehat  w)  \|_2^2}  +2 \RoundBr{\eta^\top Z_S (Z_S^\top Z_S + N\lambda I)^{-1} Z_S^\top \eta}.
    \end{aligned}
\end{equation}
Let $P_{\lambda } = Z_S (Z_S^\top Z_S + N\lambda I)^{-1} Z^\top_S$, we have that
\begin{equation}
    \begin{aligned}
        \|P_\lambda \|_2 \leq \frac{1}{\lambda} ;\quad
        \|P_\lambda \|_F  \leq \frac{\sqrt{s}}{\lambda}.
    \end{aligned}
\end{equation}
Applying Hanson-Wright inequality and take union bound for all subset $S$ such that $|S| \leq 2s$, we have that with probability  $1-\delta$, 
\begin{align*}
    \eta^\top P_\lambda \eta &= \mathbb E[\eta^\top P_\lambda \eta]  + O\RoundBr{\sigma^2 s \log(d\delta^{-1})} \\
    &= O\RoundBr{\frac{\sigma^2 \Tr(\hat \Sigma)}{\lambda} + \sigma^2 s \log(d\delta^{-1})} \\
    &= O\RoundBr{\frac{\sigma^2 s}{\lambda} + \sigma^2 s \log(d\delta^{-1})} \tag{$\|z \|_\infty \leq 1$}.
\end{align*}
\end{proof}

Therefore, 
\[
\frac{2}{N}\AngleBr{Z_{C_* \cap C}(\widehat  w -  w_\star), \eta_{C \cap C_*}}  \leq \frac{1}{2N}\|Z_{ C_* \cap C}( w_\star -\widehat  w) \|_2^2 + O\RoundBr{\frac{\sigma^2 s}{N\lambda} + \frac{\sigma^2 s \log(d\delta^{-1})}{N}}.
\]
Consider (2)

\begin{equation*}
    \begin{aligned}
        \frac{1}{N} \norm{Y_{C_* \setminus \widehat C} - Z_{C_* \setminus \widehat C}  w_\star}_2^2 = \frac{1}{N} \norm{\eta_{C_*\setminus \widehat C}}_2^2 \leq 2\sigma^2\epsilon.
    \end{aligned}
\end{equation*}

Therefore, we have that
\begin{equation}
    \frac{1}{2N} \norm{Z_{C_* \cap \widehat C}(\widehat  w -  w_\star)}_2^2 + \frac{\lambda}{2} \norm{\widehat  w -  w_\star}_2^2 \leq  O\RoundBr{\frac{\sigma^2 s}{N\lambda} + \frac{\sigma^2 s \log(d\delta^{-1})}{N}  + 2\lambda \| w_\star \|_2^2 +  \sigma^2\epsilon}.
\end{equation}
As $|C_\star \cap C|\geq (1-2\epsilon)$, and assume that $\epsilon\leq 1/4$, we have that  $ \frac{1-2\epsilon}{2(1-\epsilon)} \geq \frac 1 4$, $\frac{2}{1-\epsilon} \leq 4$, $\frac{1}{1-\epsilon}\leq 2$. Therefore,
\begin{equation}
    \norm{\widehat  w -  w_\star}_{\widehat{\Sigma}_\lambda(\widehat C\cap C_*) }^2 \leq O\RoundBr{\frac{\sigma^2 s}{N\lambda} + \frac{\sigma^2 s \log(d\delta^{-1})}{N} + 2\lambda \| w_\star \|_2^2 + \sigma^2\epsilon}.
\end{equation}
\end{proof}

Let $\Delta_{S,\lambda} = (\widehat\Sigma_{SS,\lambda})^{-1/2}(\Sigma_{SS}-\widehat\Sigma_{SS})(\widehat\Sigma_{SS,\lambda})^{-1/2}$.
\begin{lemma}\label{lem: SRLE2 from prediction to risk}
    Let $\widehat  w- w_\star$ has non-zero support in $S$. Assume $\|\Delta_{S,\lambda} \|\leq 1$, then,
    \begin{equation}
      \|\widehat  w -  w_\star \|_{\Sigma}^2 \leq  \|\widehat  w -  w_\star \|_{\Sigma_\lambda}^2  \leq \frac{1}{1-\|\Delta_{S,\lambda} \|_2}\|\widehat  w -  w_\star \|_{\widehat \Sigma_\lambda}^2.
    \end{equation}
\end{lemma}
\begin{proof}
    It is straightforward from the fact that $(\widehat  w -  w_\star )$ is $2s$-sparse vector, and for $\| v\|_2\leq1$, we have that $\frac{v^\top A v}{v^\top \sigma v} \leq \|\sigma^{-1}A\|_2$ if $A, \;\sigma$ are both strictly symmetric PSD matrix.
    And By Lemma 3 of \citep{Hsu2014_RLS_GeneralRandomDesign}, we have that
    \[
    \|\Sigma_{S,\lambda}^{1/2} \widehat \Sigma_{S,\lambda}^{-1} \Sigma_{S,\lambda}^{1/2}  \|_2  \leq \frac{1}{1-\|\Delta_{S,\lambda}\|_2}.
    \]
\end{proof}

Now, fix a subset $S$, we will drop the subscript for clearer presentation.
For $z \in \mathbb R^{2s}$,
\begin{equation}
    \begin{aligned}
        d_{1,\lambda}&:= \sum_{j=1}^{2s} \frac{\lambda_i(\Sigma_{SS})}{\lambda+\lambda_i(\Sigma_{SS})} \leq \sum_{j=1}^{2s} \frac{\lambda_i(\Sigma_{SS})}{\lambda}  \leq \frac{2s}{\lambda} \\
        \rho_\lambda &\geq \frac{\|\Sigma_{SS,\lambda}^{-1/2 }z \|}{\sqrt{d_{1,\lambda}}}\\
        \tilde{d}_{1,\lambda} &:= \max\{1, d_{1,\lambda}\}.
    \end{aligned}
\end{equation}
Note that, as $\|z \|_\infty \leq 1$, it means $\| z\|_2 \leq \sqrt{2s}$ we can choose $\rho_\lambda = \frac{\sqrt{2s}}{\sqrt{\lambda d_{1,\lambda}}}$.
\begin{lemma}[\citep{Hsu2014_RLS_GeneralRandomDesign}'s Lemma 2]
    For any $\delta$ such that $\ln(1/\delta)>\max\{0,6-\log \tilde d_{1,\lambda}\}$, then with probability at least $1-\delta$, we have that
    \begin{equation}
        \|\Delta_\lambda \|_2 \leq \sqrt{\frac{4\rho_\lambda^2 d_{1,\lambda}(\log(\tilde d_{1,\lambda}+\delta^{-1}))}{N}}
        + \frac{2\rho_\lambda^2 d_{1,\lambda}(\log(\tilde d_{1,\lambda}+\delta^{-1}))}{3n}.
    \end{equation}
    With the choice of $\rho_\lambda$ above, the bound can be further simplified as
    \begin{equation}
        \|\Delta_\lambda \|_2 \leq \sqrt{\frac{8s(\log(\tilde d_{1,\lambda}+\delta^{-1}))}{\lambda n}}
        + \frac{4s (\log(\tilde d_{1,\lambda}+\delta^{-1}))}{3\lambda n}.
    \end{equation}
\end{lemma}
Choose $\lambda  = \frac{\sigma\sqrt{s}}{\sqrt{N} \| w_\star \|_2}$ to optimise the bound in Lemma \ref{prop: Prediction error SRLE}.
Moreover, Taking union bound for all $S$, for $\|\Delta_\lambda\|_2\leq 1/2$ with probability at least $1-\delta$, it suffices to choose $N = \Omega\left(\frac{s \| w_\star \|_2^2 (\log(2s+d\delta^{-1}))}{ \sigma^2(1-2\epsilon)^2 }\right)$.

\ThmRSLSwithoutDataCoverage*
\begin{proof}
By choosing $\lambda = \tfrac{\sigma\sqrt{s}}{\sqrt{N}\,\| w_\star \|_2}$ and ensuring that 
$
N = \Omega\left(\frac{s \| w_\star \|_2^2 \log^2(2s+d\delta^{-1})}{(1-2\epsilon)^2}\right),
$
we can guarantee that for all $S \subseteq [d]$ with $|S|\leq 2s$, 
$
\Delta_{S,\lambda} \;\geq\; \tfrac{1}{2}
$
with probability at least $1-\delta$.
 Combining the result in Lemma \ref{lem: SRLE2 from prediction to risk} and Lemma \ref{prop: Prediction error SRLE}, we conclude the proof.
\end{proof}

\subsection{Estimation Error of \texttt{SRLE3}}\label{appendix: SRLE3}
Similar as \texttt{SRLE1}, \texttt{\textbf{SRLE3}} was proposed in \citep{Merad2022_RobustSparseLinearRegressionHeavyTailDist} to solve the convex optimization problem
\begin{equation}
\begin{aligned}
        \min_{w \in \mathbb{R}^d}\; & \|Xw - Y\|_2^2  \\
        \text{s.t.}\; & \|w \|_1 \leq \|w_\star \|_1,
\end{aligned}
\end{equation}
Without the uniform coverage assumption, the mirror descent algorithm with a suitable choice of Bregman divergence function is still able to achieve the slow rate. 
For further details, we refer the reader to the original paper \citep{Merad2022_RobustSparseLinearRegressionHeavyTailDist}. 
We now state the error guarantee of this estimator in our setting.
\PropRSLSwithoutDataCoverageCompEfficient*
\begin{proof}
    By Proposition 3 in \citep{Merad2022_RobustSparseLinearRegressionHeavyTailDist}, we have that
    \begin{equation}
        \|\hat  w -  w_\star \|_\Sigma^2 \leq  \frac{\nu R L}{ T} 
        + 4\| w_\star\|\sigma_{\max}\RoundBr{\sqrt{\frac{\log(d) + \log(\delta^{-1})}{N}}+\sqrt{\epsilon}}.
    \end{equation}
    Where $\nu = \frac{1}{2}e^2\log(d)$; 
    and $L=1$ is the Lipschitz-smoothness, that is,
    \[\|\Sigma ( w- w') \|_\infty =  \max_{i\in [d]} \AngleBr{\Sigma_i, w - w'} \leq \max_{i\in [d]} \|\Sigma_i\|_\infty \| w - w'\|_1 \leq  L \| w - w'\|_1. \]
    and $\nu = \frac{1}{2}e^2\log(d)$.
    Moreover,
    \[
    \sigma_{\max}^2 \triangleq \max_{ w: \|  w\|_1\leq \| w_\star \|_1} \max_{j\in [d]} \mathbb E_{z,y}\SquareBr{((z^\top  w -y)x_j)^2}.
    \]
    We note that, since $\|z\|_\infty \leq 1$
    \[
    \sigma_{\max}^2 \leq \max_{ w}\mathbb E_{z,y}\SquareBr{(z^\top  w -y)^2} \leq \max_{ w}\max_{z} \mathbb E_{z,y}\SquareBr{(z^\top ( w - w_\star))^2} + \mathbb E[\eta^2] = 4\| w_\star\|_1^2 + \sigma^2.
    \]
\end{proof}
\section{Proof of Section \ref{section: Why Robust LSVI may Fail in Sparse Offline RL}}
\subsection{Proof of Section \ref{subsec: Sparse Robust LSVI with Uniform Coverage}}
\PropSubOptUniformCoverageLSVI*
\begin{proof}
    First, by Claim \ref{claim: norm of true vector}, for any $Q_{h+1} \in \mathcal Q_{h+1}$, $\mathcal P_h^\pi(Q_{h+1})$ is $s$-sparse, and $\|\mathcal P_h^\pi(Q_{h+1})\|_1\leq (H-h+1)$.
    Moreover, by Claim \ref{claim: bound of output}, for any trajectory $\tau$, we have $\left|R_h + \max_{a'\in \mathcal A} Q_{h+1}(x',a')\right| \leq H$.
    Note that $|\mathcal D_h | = N/H$ and the worst-case contamination level in $\mathcal D_h $ is $H\epsilon$. Therefore, using the estimation error in Proposition \ref{prop: RSLS with data coverage} with $\|w \|_1 \leq 2Hs$ and $\sigma \leq H$, we obtain
    \begin{equation}
        \|\mathcal R^*_h(\underline Q_{h+1}) -  \mathcal P^*_h(\underline Q_{h+1})\|_1 = O\left(\frac{H^{\tfrac{3}{2}}s \log(d/\delta)}{\xi \sqrt{N}} + \frac{H^2 s \sqrt{\epsilon}}{\xi}\right).
    \end{equation}
    
 Similar to Lemma 3.1 of \cite{Zhang2021_RobustOfflineLinMDP}, we have that
     \begin{equation*}
        \begin{aligned}
            \mathrm{SubOpt}(\pi,\widehat \pi) &\leq 2H \max_{(x,a,h)} \left|\underline{Q}_h(x,a) - (\mathbb B_h\underline{Q}_{h+1})(x,a)\right| \\
           &\leq 2H \norm{\phi(x,a)}_\infty \norm{\mathcal R^*_h(\underline Q_{h+1}) -  \mathcal P^*_h(\underline Q_{h+1})}_1 \\
            &= O\left(\frac{H^{3}s \log(d/\delta)}{\xi \sqrt{N}} + \frac{H^3s \sqrt{\epsilon}}{\xi}\right).
        \end{aligned}
    \end{equation*}
 
\end{proof}

\subsection{Proof of Section \ref{subsec: Sparse Robust LSVI with Single Concentrability}}

\LemmaMaxExpectation*
\begin{proof}
For any subset \(S\subset[d]\) with \(|S|=2s\) define
\[
Z_S\;:=\;\sum_{i\in S}z_i
\qquad(\text{so }Z_S=z_S^{\top}z_S).
\]

\textbf{Step 1.}
Because each \(z_i\in\{0,1\}\),
\(Z_S\) counts how many of the \(2s\) chosen coordinates
are equal to 1.
For a realisation of \(z\) let \(K=\sum_{i=1}^{d}z_i\)
(the \emph{total} number of ones).
If \(K\ge 2s\) we can pick
all \(2s\) ones, so \(\max_{|S|=2s}Z_S=2s\).
If \(K<2s\) we pick all \(K\) ones and fill the remaining slots with zeros, so the maximum equals \(K\).
Thus,
\[\max_{|S|=2s}Z_S=\min\{2s,K\}.\]

\textbf{Step 2.}
Write \(\mu=d/2\) for the mean of \(K\).
Because \(2s\le d/2=\mu\), Chernoff’s inequality gives
\[
\mathbb P(K<2s)\;=\;\mathbb P\bigl(K<(1-\delta)\mu\bigr)
\;\le\;
\exp(-\tfrac{\delta^{2}\mu}{2}),
\quad
\delta:=1-\tfrac{4s}{d}\;\in(0,1].
\]
In particular, for all \(d\ge 4s\) we have \(\delta\ge\tfrac12\), hence
\[
\mathbb P(K<2s)\;\le\;e^{-d/8}.
\]

\textbf{Step 3.}
Using \(\mathrm{(1)}\) we decompose
\[
\mathbb{E}\SquareBr{\max_{|S|=2s}Z_S}
= 2s\,\mathbb P(K\ge 2s)\;+\;\mathbb{E}\SquareBr{K\;\mathbf 1_{\{K<2s\}}} \geq 2s(1-e^{-d/8}).
\]

If we use the ridge matrix
\(\Sigma=\lambda I_d\) (so that
\(\Sigma^{-1}=\tfrac{1}{\lambda} I_d\)),
then \(z_S=\tfrac 1 \lambda Z_S\).
Multiply inequality \(\mathrm{(2)}\) by \(\lambda\):

\[
\mathbb{E}\SquareBr{\max_{|S|=2s}z_S}
\;-\;
\max_{|S|=2s}\mathbb{E}[z_S]
\geq \lambda s\bigl(1-2e^{-d/8}\bigr) 
;\qquad(d\geq 4s)
\]
\end{proof}

\section{Proof of Section \ref{section: Sparse Actor-Critic Methods}}
\paragraph{Induced MDP.}
To prove the results stated in Section \ref{section: Sparse Actor-Critic Methods}, we use the notion of an induced MDP with perturbed rewards \citep{zanette2021_provablebenefitsactorcriticmethods}. 
This perspective allows us to interpret the critic’s pessimistic estimates as defining a modified MDP, in which the value functions coincide exactly with the pessimistic predictions.  

Define the induced MDP $\widehat M(\pi) = (\mathcal X, \mathcal A, P, \widehat r^\pi, H)$ associated with ${\{\underline{w}_h\}_{h=1}^H}$, which differs from the original MDP $M$ only in its reward function. 
For any $(x,a)$, define
\begin{equation}
    \widehat r_h^\pi(x,a) \triangleq r_h(x,a) + \underline{Q}_h^\pi(x,a) - (\Bell^\pi_h \underline{Q}_{h+1}^\pi)(x,a).
\end{equation}
This construction is crucial: the linearity of the $\pi$-Bellman operator guarantees that the induced MDP is well-defined and enables tight pessimistic evaluation without the need for pointwise bonuses.

Below, we demonstrate that the induced MDP corresponds to the pessimistic evaluation of $Q$-function provided by the critic.
\begin{restatable}{proposition}{PropInducedMDP} \label{prop: Induced MDP}
    $\widehat M(\pi)$ satisfies the following: $Q_{h,\widehat M^\pi }(x,a) = \underline{Q}_h^\pi(x,a)$, and $V_{h,\widehat M^\pi }(x) = \underline{V}_h^\pi(x).$
\end{restatable}
\begin{proof}
We begin by bounding the pessimistic property of the value function at the initial state:
\begin{equation}
    Q_{h,\widehat M^\pi }^\pi(x,a) - Q_{h}^\pi(x,a) 
    = \sum_{l=h}^H \mathbb E_{d^\pi_l} [\widehat r^\pi(x_l,a_l) - r(x_l,a_l)].
\end{equation}
On the other hand, using the definition of $\underline{Q}_{h}^\pi$ and the Bellman operator, we obtain
\begin{equation}
\begin{aligned}
    \underline{Q}_{h}^\pi(x,a) - Q_{h}^\pi(x,a) 
    &= \langle \phi(x,a), \underline w_h^\pi \rangle - \Bell^\pi_h(Q_{h+1}^\pi)(x,a) \\
    &= \RoundBr{\langle \phi(x,a), \underline w_h^\pi \rangle - \Bell^\pi_h(\underline{Q}_{h+1}^\pi)(x,a)} 
      + \RoundBr{\Bell^\pi_h(\underline{Q}_{h+1}^\pi)(x,a) - \Bell^\pi_h(Q_{h+1}^\pi)(x,a)} \\
    &= \widehat r^\pi_h(x,a) - r_h(x,a) 
       + \mathbb E_{x'\sim P(\cdot\mid x,a)}\mathbb E_{a \sim \pi(\cdot\mid x)}[\underline{Q}_{h+1}^\pi - Q_{h+1}^\pi](x,a).
\end{aligned}
\end{equation}
Here, the last equality follows from the linearity of the Bellman operator $\Bell_h^\pi$.

Applying this argument recursively for $l=1,\dots,H$, we obtain
\begin{equation}
    \underline{Q}_{h}^\pi(x,a) - Q_{h}^\pi(x,a) 
    = \sum_{l=h}^H \mathbb E_{d^\pi_l} [\widehat r_l^\pi(x_l,a_l) - r_l(x_l,a_l)].
\end{equation}
This establishes the claim and concludes the proof.
\end{proof}

Next, we recall the general theorem for actor convergence \citep{zanette2021_provablebenefitsactorcriticmethods}, which will be used in our proof.
\begin{theorem}[\citep{zanette2021_provablebenefitsactorcriticmethods}] 
\label{theorem: general actor convergence}
Let $M_t \triangleq \hat M(\pi_t)$. 
For each $M_t$, define the advantage function
\[
    G_{h,M_t}^{\pi_t}(x,a) \triangleq Q_{h,M_t}^{\pi_t}(x,a) - V_{h,M_t}^{\pi_t}(x).
\]
Assume that the advantage function is uniformly bounded, i.e.,
\[
    \max_{t\in [T]} \max_{(x,a,h)} \, |G_{h,M_t}^{\pi_t}(x,a)| \leq B.
\]
Suppose further that $T \geq \log(|\mathcal A|)$ and that the step size satisfies $\eta \in (0,1)$. 
Then, for any fixed policy $\pi$, we have
\begin{equation}
    \frac{1}{T} \sum_{t=1}^T \RoundBr{V_{1,M_t}^{\pi_t}(x_1) - V_{1,M_t}^{\pi}(x_1)} 
    = O\RoundBr{ H \RoundBr{\frac{\log(|\mathcal A|)}{\eta T} + \eta B^2} }.
\end{equation}  
\end{theorem}

\subsection{Proof of Section \ref{subsec: Sparse Robust AC with Uniform Coverage}}
For a sequence $\beta = (\beta_h)_{h=1}^H$ and any given policy $\pi$, define the good event
\begin{equation}
    \mathcal G'^\pi(\beta) \triangleq \CurlyBr{ \sup_{Q_{h+1} \in \mathcal Q_{h+1}} \|\mathcal E_{h}^{\pi}(Q_{h+1}) \|_1 \leq \beta_h,\; \forall h \in [H]}.
\end{equation}
Under the good event $\mathcal G'^\pi(\beta)$ for the sequence of policies produced by the actor, we will show that the suboptimality gap is small.
Before doing so, we first establish that, under the good event $\mathcal G'^\pi(\beta)$, the critic’s estimation error can be bounded as follows.

\begin{proposition} \label{prop: Policy Evaluation Uniform Coverage}
    Conditioned on the event $\mathcal G'^\pi(\beta)$, for some input policy $\pi$, 
    the critic returns an induced MDP $\hat M(\pi)$ such that, for every policy $\tilde \pi$,
    \begin{equation}
        \left|V^{\tilde \pi}_{1,\hat M(\pi)}(x_1) - V_1^{\tilde \pi}(x_1)\right| 
        \leq 2 \sum_{h=1}^H \beta_h.
    \end{equation}
\end{proposition}

\begin{proof}
Let $\hat w_h^\pi = \mathcal R^\pi(\underline{Q}_{h+1}^\pi)$.
We bound $\underline w_h^\pi - \mathcal P_h^\pi(\underline{Q}_{h+1}^\pi)$ as follows:
\begin{align*}
    \norm{\underline w_h^\pi - \mathcal P_h^\pi(\underline{Q}_{h+1}^\pi)}_1 
    &\leq \norm{\underline w_h^\pi - \hat w_h^\pi}_1 
        + \norm{\hat w_h^\pi - \mathcal P_h^\pi(\underline{Q}_{h+1}^\pi)}_1  \\
    &\leq 2 \norm{\hat w_h^\pi - \mathcal P_h^\pi(\underline{Q}_{h+1}^\pi)}_1 
        \tag{Conditioned on $\mathcal G'^\pi(\beta)$, $\hat w_h^\pi$ is a feasible solution of \eqref{eq: Critic with uniform coverage}} \\
    &\leq 2 \beta_h.
\end{align*}

Therefore, for all $h \in [H]$,
\begin{equation*}
\begin{aligned}
    \max_{(x,a)} \hat r^\pi_h(x,a) - r_h(x,a) 
    &= \max_{(x,a)} \left|\underline{Q}_h^\pi(x,a) - (\Bell^\pi \underline{Q}_{h+1}^\pi)(x,a)\right|  \\
    &\leq \AngleBr{\underline{w}_h^\pi - \mathcal P^\pi(\underline{Q}_h^\pi), \phi(x,a)} \\
    &\leq \norm{\underline{w}_h^\pi - \mathcal P^\pi(\underline{Q}_h^\pi)}_1 \, \|\phi(x,a)\|_\infty \\
    &\leq 2\beta_h .
\end{aligned}
\end{equation*}

Therefore, we have that
\begin{equation*}
    \begin{aligned}
        V^{\tilde \pi}_{1,\hat M(\pi)}(x_1) -  V_1^{\tilde \pi}(x_1) = \sum_{h=1}^H \mathbb E_{d^{\tilde \pi}}[\widehat r^\pi_h(x,a) - r_h(x,a)]  
        \leq\sum_{h=1}^H\RoundBr{\max_{(x,a)}\hat r^\pi_h(x,a) - r_h(x,a)}
        \leq 2\sum_{h=1}^H \beta_h.
    \end{aligned}
\end{equation*}
This completes the proof.
\end{proof}

\TheoremSuboptimalGapACUC*
\begin{proof}
\textbf{Bound probability of good event $\mathcal G'(\beta)$.}
choose the sequence $\beta$ 
\[
\beta_h =  O\RoundBr{\frac{H^2s \log(dHN\delta^{-1})}{\xi \sqrt{N}} + \frac{H^2s\sqrt{\epsilon}}{\xi}}.
\]
According to Proposition \ref{prop: RSLS with data coverage}, and by Claim \ref{claim: norm of true vector} and Claim \ref{claim: bound of output}, we have that for any policy $\pi$, $\mathbb P(\mathcal G'^\pi(\beta))\geq 1-\delta/{T}$.  
Therefore, for any sequence of policy $(\pi_t)_{t=1}^T$, taking union bound for $T$ policies and $H$ horizon, we have that,
$\mathbb P(\cap_{t=1}^T\mathcal G'^{\pi_t} (\beta))\geq 1-\delta$.

\textbf{Suboptimality gap decomposition.}    
Let $M_t \triangleq \hat M(\pi_t)$. From the critic analysis, under event $\mathcal G^{\pi_t}(\beta)$, 
    \begin{equation}
    \begin{aligned}
        V_1^{\pi_\star}(x_1) - V_1^{\pi_t}(x_1) & = 
        \RoundBr{V_1^{\pi_\star}(x_1) - V^{\pi_\star}_{1,M_t}(x_1)} + \RoundBr{V_{1,M_t}^{\pi_\star}(x_1) - V^{\pi_t}_{1,M_t}(x_1)} + \RoundBr{V_{1,M_t}^{\pi_t}(x_1) - V_1^{\pi_t}} \\
        &\leq V_{1,M_t}^{\pi_\star}(x_1) - V_{1, M_t}^{\pi_t}(x_1) +  2 \sum_{h=1}^H\beta_h.
    \end{aligned}
    \end{equation}

\textbf{Actor Convergence.}
We now bound 
\[
\frac{1}{T}\sum_{t=1}^T \SquareBr{V_{1,M_t}^{\pi_\star}(x_1) - V_{1, M_t}^{\pi_t}(x_1)}.
\]
To do so, we invoke Theorem \ref{theorem: general actor convergence} and provide an upper bound for the advantage function parameter $B$, namely $B = O(H)$.

By Proposition \ref{prop: Induced MDP}, and assuming the event $\mathcal G^{\pi_t}(\beta)$ holds for the sequence $(\pi_t)_{t=1}^T$ with $\beta_h \leq 1$, we obtain
\[
\left|\hat{r}_h^{\pi_t}(x,a) - \hat{r}_h^{\pi_t}(x,a')\right| 
\leq |r_h(x,a) - r_h(x,a')| + 2\beta_h \leq 4.
\]
Moreover, from the constraints in \eqref{eq: Critic with uniform coverage}, for any $(x,h)$ we have
\[
V_{h+1, M_t}^{\pi_t}(x) \leq H-h.
\]

Therefore, for any triplet $(x,a,h)$,
\begin{align*}
    \left|Q_{h,M_t}^{\pi_t}(x,a) - Q_{h,M_t}^{\pi_t}(x,a') \right| 
    &\leq \left|\hat{r}_h^{\pi_t}(x,a) - \hat{r}_h^{\pi_t}(x,a')\right| 
         + \left| \int_{x'\in \mathcal X } V_{h+1, M_t}^{\pi_t}(x') 
            \big(P(x'\mid x,a)-P(x'\mid x,a')\big) \right| \\
    &\leq O(1) + \left|\AngleBr{\phi(x,a) - \phi(x,a'),  
            \int_{x'\in \mathcal X} V_{h+1, M_t}(x') \mu(x')} \right| \\
    &\leq O(1) + \|\phi(x,a) - \phi(x,a')\|_\infty 
            \norm{\int_{x'\in \mathcal X} V_{h+1, M_t}(x') \mu(x')}_1 \\
    &\leq O(1) + 2\norm{\int_{x'\in \mathcal X} V_{h+1, M_t}(x') \mu(x')}_1 \\
    &\leq O(1) + 2\norm{H \int_{x'\in \mathcal X}\mu(x')}_1 \\
    &\leq O(1) + 2H.
\end{align*}
Thus, we can take $B = O(H)$.

Applying Theorem \ref{theorem: general actor convergence}, we conclude that
\begin{equation}
    \frac{1}{T}\sum_{t=1}^T \SquareBr{V_{1,M_t}^{\pi_\star}(x_1) - V_{1, M_t}^{\pi_t}(x_1)} 
    \leq 4H^2 \sqrt{\frac{\log(|\mathcal A|)}{T}}.
\end{equation}

\textbf{Putting things together.} Let $\hat \pi$ be the mixture of policies $\{\pi_1,...,\pi_T\}$ and $T= N/H$, we have
\begin{equation}
            V_1^{\pi_\star}(x_1) - V_1^{\pi_t}(x_1) \leq 2\sum_{h=1}^H\beta_h +  4H^3 \sqrt{\frac{\log \mathcal A}{N}}
\end{equation}
with probability at least $1-\delta$.
Finally, by the choice of $\beta$, we have that
\begin{equation}
    V_1^{\pi_\star}(x_1) - V_1^{\pi_t}(x_1) = O \RoundBr{\frac{H^3s\log(dNH\delta^{-1})}{\xi \sqrt{N}} +  H^3 \sqrt{\frac{\log (|\mathcal A|)}{N}} + \frac{H^3s\sqrt{\epsilon}}{\xi}} , 
\end{equation}
with probability at least $1-\delta$.
\end{proof}

\subsection{Proof of Section \ref{subsec: Sparse Robust AC without Uniform Coverage}}
First, we recall the definition of the good event used in this section, and show that it holds with high probability. 
For a sequence $\alpha = (\alpha_h)_{h=1}^H$ and any policy $\pi$, define
\begin{equation}
    \mathcal G^\pi(\alpha) \triangleq 
    \CurlyBr{ \sup_{Q_{h+1} \in \mathcal Q_{h+1}} 
    \|\mathcal E_{h}^{\pi}(Q_{h+1}) \|_{\hat\Sigma_h}^2 \leq \alpha_h, \;\; \forall h \in [H]}.
\end{equation}
\begin{proposition}
    Run Algorithm \ref{alg: actor-critic} with the \texttt{SRLE2} and \texttt{SRLE3} estimators.  
    Choose the sequences $\alpha$ and $\alpha'$ for these two estimators as follows:
    \[
        \alpha_h^2 = O\RoundBr{\frac{\sqrt{s}H^3\log(dH\delta^{-1})}{\sqrt{N}} + H^3\epsilon + H^3(\lambda +\epsilon)},
    \]
    and 
    \[
        {\alpha'}_h^2 = O\RoundBr{\frac{H^3\log(dH\delta^{-1})}{\sqrt{N}} + H^3\sqrt{\epsilon} + H^3(\lambda +\epsilon)}.
    \]
    Then, for any policy $\pi$,
    \[
        \mathbb P(\mathcal G^\pi(\alpha)) \geq 1-\delta,
    \]
    when using the \texttt{SRLE2} estimator, and 
    \[
        \mathbb P(\mathcal G^\pi(\alpha')) \geq 1-\delta,
    \]
    when using the \texttt{SRLE3} estimator.
\end{proposition}

\begin{proof}
    We first prove the result for the \texttt{SRLE2} estimator.  
    The proof for the \texttt{SRLE3} estimator follows the same argument, with Proposition \ref{theorem: RSLS without data coverage} replaced by Proposition \ref{prop: RSLE mirror descent ill-condition covariate}.

    \medskip
    We apply the property of the sparse linear regression oracle.  
    By Claim \ref{claim: norm of true vector}, for any policy $\pi$,
    \begin{equation}
        \begin{aligned}
            \left\| \mathcal P_h^\pi(Q_{h+1}) \right\|_1 \leq H-h,
            \qquad 
            \|\mathcal P_h^\pi(Q_{h+1}) \|_0 \leq s.
        \end{aligned}
    \end{equation}

    Let $z_h^\tau = \phi(x_\tau, a_\tau)$ and define
    \[
        y_h^\tau = r_h^\tau + \sum_{a\in \mathcal A}\pi(a\mid x^\tau_{h+1})Q_{h+1}(x^\tau_{h+1}, a).
    \]
    Then, by Claim \ref{claim: bound of output},
    \[
        \mathbb E[y_h^\tau] = \langle  \phi(x_\tau, a_\tau), \mathcal P_h^\pi(Q_{h+1})  \rangle,
        \qquad 
        |y_h^\tau| \leq H-h.
    \]

    Note that dual to data spliting, the worst case contamination level in set $\mathcal D_h$ is $H\epsilon$. By Proposition \ref{theorem: RSLS without data coverage}, for all $h\in [H]$, the \texttt{SRLE2} oracle ensures that
    \[
        \|\mathcal E_{h}^\pi(Q_{h+1}) \|_{\Sigma_h}^2 
        = O\RoundBr{\frac{H^3 \sqrt{s}\log(dH/\delta)}{\sqrt N}+ H^3\epsilon },
    \]
    with probability at least $1-\delta$.

    \begin{claim}
        \begin{equation}
            \|\mathcal E_{h}^\pi(Q_{h+1}) \|_{\hat \Sigma_h}^2 
            = O\RoundBr{\|\mathcal E_{h}^\pi(Q_{h+1}) \|_{\Sigma_h}^2 + H^3(\epsilon +\lambda )}.
        \end{equation}
    \end{claim}

    \begin{proof}[Proof of claim]
        Let $e =\mathcal E_{h}^\pi(Q_{h+1})$. Note that, due to data splitting, $Q_{h+1}$ does not dependent on data set $\mathcal D_h$.
        We have $\|e\|_1 = O(H-h)$ and $\|e\|_0 \leq 2s$.  
        Let $S$ denote the sparsity support of $e$. Then
        \begin{equation}
        \begin{aligned}
            \|e\|_{\hat \Sigma_h}^2 
            &=   e^\top\hat \Sigma_h e \\
            &=  e^\top_S\SquareBr{\frac{H}{N}\sum_{i=1}^{N/H} \phi_i \phi_i^\top + (\lambda+\epsilon) I} e_S \\
            &= e^\top_S\SquareBr{\underbrace{\frac{H}{N}\sum_{i=1}^{(1-\epsilon)N/H} \tilde \phi_i \tilde \phi_i^\top}_{(1-\epsilon)\tilde \Sigma_h , \text{ clean data}} 
            + \underbrace{\frac{H}{N}\sum_{i=1}^{\epsilon N/H}  \phi_i'  \phi_i'^\top}_{\text{corrupted data}} 
            + (\lambda+\epsilon) I} e_S .
        \end{aligned}
        \end{equation}

        For the corrupted data,
        \[
            \frac{H}{N}\sum_{i=1}^{\epsilon N/H} e^\top\phi_i'  \phi_i'^\top e
            \leq \frac{H}{N}\sum_{i=1}^{\epsilon N/H}(\|\phi_i'\|_\infty\|e\|_1)^2
            \leq \epsilon\|e \|_1^2 
            \leq O((H-h)^2\epsilon).
        \]

        For the clean data, by Lemma \ref{lem: Regularised covariance concentration on sparse supports}, for $0<\epsilon< 1/2$,
        \[
            \SquareBr{\underbrace{\frac{H}{N}\sum_{i=1}^{(1-\epsilon)N/H} \tilde \phi_i \tilde \phi_i^\top}_{(1-\epsilon)\tilde \Sigma_h , \text{ clean data}} + (\lambda+\epsilon) I}_S 
            = O([\Sigma_h+(\lambda +H\epsilon)I]_S).
        \]

        Therefore,
        \[
            \|e\|_{\hat \Sigma_h}^2 
            = O(\|e \|_{\Sigma_h}^2 + H^3(\lambda +\epsilon)).
        \]
    \end{proof}

    Combining the results completes the proof of the proposition.
\end{proof}

\begin{restatable}[Policy evaluation]{proposition}{PropPolicyEvaluation} \label{prop: Policy Evaluation}
    Condition on the event $\mathcal G^\pi(\alpha)$, when given any policy $\pi$, the critic returns an induced MDP $\hat M(\pi)$ such that: \hfill \break
    \emph{[a]} For the given policy $\pi$, we have
        \begin{equation}
            V^\pi_{1,\hat M(\pi)}(x_1) \leq V_1^\pi(x_1).
        \end{equation}
    \emph{[b]} For any policy $\tilde \pi$, we have
        \begin{equation}
        \begin{aligned}
            \Big|V^{\tilde \pi}_{1,\hat M(\pi)}&(x_1) -  V_1^{\tilde \pi}(x_1)\Big| 
            \leq 2\sum_{h=1}^H \alpha_h \mathbb E_{d^{\tilde \pi}}\SquareBr{\norm{[\phi(x,a)]_{\tilde S_h}}_{\hat\Sigma_h^{-1}} }.
        \end{aligned}
        \end{equation}
\end{restatable}
\begin{proof}
\textbf{Part a:} We first show that the ground truth $\tilde  w^{\pi}_h$, that is, the ground truth $Q^\pi_h(x,a) = \langle \phi(x,a), \tilde w^\pi_{h} \rangle$ , satisfies the program.
In particular, let $\tilde w^\pi_{H+1} = \bm 0$ and
$\tilde w^\pi_{h} = \mathcal P_h^\pi(\langle \phi, \tilde w^\pi_{h+1} \rangle).$
By Claim \ref{claim: norm of true vector}, we have that $\|\tilde  w_h^\pi \|_1 \leq H-h$.

Next, under the event $\mathcal G^\pi(\alpha)$, and note that ground truth $Q_{h+1}^\pi \in \mathcal Q_{h+1}$, we have that
\[
\|\tilde  w_h^\pi - \mathcal R_h^\pi(Q_{h+1}^\pi) \|_{\hat \Sigma_h}^2 \leq \alpha_h.
\]
Which means that $(\tilde  w^{\pi}_h)_{h=1}^H$ is a feasible solution of the optimization problem.
The output of the optimization problem $\underline  w_h^\pi$ satisfies
\begin{equation}
    V_{1,M(\pi)}^\pi(x_1) =\underline{V}_1^\pi(x_1) \leq V^\pi_{1}(x_1).
\end{equation}

\textbf{Part b:}
First, from the definition of perturbed reward of $M(\pi)$, we have
\begin{equation}
    \begin{aligned}
        \hat r_h^\pi(x,a) - r_h(x,a) &= \langle \phi_h(x,a), \underline  w_h^\pi \rangle - \Bell^\pi_h(\underline Q_{h+1}^\pi)(x,a) \\
        &= \langle \phi_h(x,a), \underline  w_h^\pi - \mathcal R(\underline Q_{h+1}^\pi)\rangle + \mathcal R(\underline Q_{h+1}^\pi - \Bell^\pi_h(\underline Q_{h+1}^\pi)(x,a) \\
        & = \langle \phi_h(x,a), \underline  w_h^\pi - \mathcal R(\underline Q_{h+1}^\pi)\rangle +  \langle \phi_h(x,a),
        \mathcal R(\underline Q_{h+1}^\pi) - \mathcal P(\underline Q_{h+1}^\pi)\rangle \\
        &\leq 2 \|[\phi(x,a)]_{\tilde S} \|_{\hat\Sigma_h^{-1}} \alpha_h.\\
    \end{aligned}
\end{equation}
Therefore, 
\begin{equation}
    \begin{aligned}
        \left|V_{1,\hat M(\pi)}^{\tilde \pi}(x_1) -  V_{1}^{\tilde \pi}(x_1)\right| 
        &= \left| \sum_{h=1}^H \mathbb E_{ d^{\tilde \pi}_h}[\hat r_h^\pi(x,a) - r_h(x,a)] \right| \\
        &\leq \left| 2\sum_{h=1}^H \mathbb E_{ d^{\tilde \pi}_h}\SquareBr{\alpha_h \norm{[\phi(x,a)]_{\tilde S} }_{ \hat\Sigma_h^{-1}} } \right|. \\
    \end{aligned}
\end{equation}
\end{proof}

\begin{restatable}[Actor's convergence]{proposition}{PropositionActorConvergence} \label{prop: Actor Convergence}
    For the sequence $(\pi_t)_{t=1}^T$, suppose that the event $\mathcal G^{\pi_t}(\alpha)$ holds for all $t\in [T]$.
    Suppose the actor takes $T \geq \log |\mathcal A|$ steps with step size $\eta = 
 \sqrt{\frac{\log(|\mathcal A|)}{TH^2}}$. Then, for any  policy $\pi$, 
    \begin{equation*}
        \frac{1}{T} \sum_{t=1}^T\RoundBr{V_{1,M_t}^\pi(x_1) - V_{1,M_t}^{\pi_t}(x_1)} = O\RoundBr{H^2\sqrt{\frac{\log(|\mathcal A|)}{T}}}.
    \end{equation*}
\end{restatable}
\begin{proof}
    The proof is Proposition \ref{prop: Actor Convergence} followed by the general result for offline actor-critic as stated in Theorem \ref{theorem: general actor convergence}.
    The Proposition \ref{prop: Actor Convergence} is immediately followed if we can show the constant in the theorem $B = O(H)$  and choosing the stepsize $\eta$.
    First, we bound the magnitude of the advantage function $G_{h,M_t}^{\pi_t}$.
    Note that by the definition of induced MDP, $M_t$ is only different from the original MDP in the reward function.
    By Proposition \ref{prop: Induced MDP}, and suppose the event $\mathcal G^{\pi_t}(\alpha)$ holds for given sequence $(\pi_t)_{t=1}^T$ with $\alpha_h \leq 1$, we have that 
    \[
    \left|\hat{r}_h^{\pi_t}(x,a) - \hat{r}_h^{\pi_t}(x,a')\right| \leq  |r_h(x,a) - r_h(x,a')| + 2\alpha_h \leq 4.
    \]
    Moreover, thanks to the pessimistic property of induced MDP in Proposition \ref{prop: Policy Evaluation}, for any $(x,h)$,  $V_{h, M_t}^{\pi_t}(x) \leq V_{h}^{\pi_t}(x) \leq H-h+1$.
    Therefore, for any triplet $(x,a,h)$, we have that
        \begin{align*}
            \left|Q_{h,M_t}^{\pi_t}(x,a) - Q_{h,M_t}^{\pi_t}(x,a') \right| &
            \leq \left|\hat{r}_h^{\pi_t}(x,a) - \hat{r}_h^{\pi_t}(x,a')\right| + \left| \int_{x'\in \mathcal X }V_{h+1, M_t}^{\pi_t}(x') (P(x'\mid (x,a)-P(x'\mid (x,a')) \right| \\
            &\leq O(1) + \left|\AngleBr{\phi(x,a) - \phi(x,a'),  \int_{x'\in \mathcal X }V_{h+1, M_t} (x') \mu(x')} \right| \\
            &\leq O(1) + \|\phi(x,a) - \phi(x,a')  \|_\infty \norm{\int_{x'\in \mathcal X }V_{h+1, M_t} (x') \mu(x') }_1 \\
            &\leq O(1) + 2\norm{\int_{x'\in \mathcal X }V_{h+1, M_t} (x') \mu(x') }_1 \\
            &\leq O(1) + 2\norm{H \int_{x'\in \mathcal X }\mu(x')}_1 \\
            &\leq O(1) + 2H.
        \end{align*}
        Therefore, we can choose $B = O(H)$.
        Now, choose $\eta = \frac{\sqrt{\log(|\mathcal A)|}}{H\sqrt{T}}$, we obtain the result.
\end{proof}
\ThmMainSubOptGapActorCritic*
\begin{proof}
    We note that, for any sequence of $T$ policy $(\pi_t)_{t=1}^T$, with prescribed $\alpha$, the event $\mathbb P(\cap_{t=1}^T\mathcal G^{\pi_t}(\alpha)) \geq 1-\delta$. 
    
    Next, from the critic analysis, under event $\cap_{t=1}^T\mathcal G^{\pi_t}(\alpha)$, we have that 
    \begin{equation}
    \begin{aligned}
        V_1^{\pi_\star}(x_1) - V_1^{\pi_t}(x_1) &\leq \RoundBr{V_{1,M_t}^{\pi_\star}(x_1) + 2 \sum_{h=1}^H \alpha_h \mathbb E_{ d^{\pi_\star}}\SquareBr{\norm{[\phi(x,a)]_{\tilde S_h}}_{\hat\Sigma_h^{-1}} }} - V_{1, M_t}^{\pi_t}(x_1) \\
        &= V_{1,M_t}^{\pi_\star}(x_1) - V_{1, M_t}^{\pi_t}(x_1) +  2 \sum_{h=1}^H \alpha_h \mathbb E_{d^{\pi_\star}} \SquareBr{\norm{\phi(x,a)]_{\tilde S_h}}_{\hat\Sigma_h^{-1}} }.
    \end{aligned}
    \end{equation}
    As $\epsilon\le 1/2$, by Lemma \ref{lem: Regularised covariance concentration on sparse supports}, with number of samples $N= \Omega(s\log(d/\delta))$, then with probability at least $1-\delta$, for all $\tilde S \subset [d]$ such that $|\tilde S| \leq 2s$, we have that
    \[
    [\hat\Sigma_h]_{\tilde{S}} \geq \frac{1}{3}([\Sigma_h + \lambda I]_{\tilde S}).
    \]
    Now, consider 
    \begin{equation}
        \begin{aligned}
            \mathbb E_{d^{\pi_\star}} [\|[\phi(x,a)]_{\tilde S_h}\|_{\hat\Sigma_h^{-1}} ] 
            &\leq  \mathbb E_{d^{\pi_\star}} [\|[\phi(x,a)]_{\tilde S_h}\|_{3(\Sigma_h + \lambda I)^{-1}} ] \\
            &\leq \mathbb E_{d^{\pi_\star}} \SquareBr{\sqrt{3\phi_{\tilde S}(\Sigma_h + \lambda I)^{-1} \phi_{\tilde S}} } \\
            &\leq \sqrt{\mathbb E_{ d^{\pi_\star}} \SquareBr{3\phi_{\tilde S}(\Sigma_h + \lambda I)^{-1} \phi_{\tilde S}}} \\
            &= \sqrt{3\Tr([\Sigma_* \hat\Sigma_h^{-1}]_{\tilde S})} \\
            &\leq \sqrt{3\kappa\Tr([\Sigma_h \hat\Sigma_h^{-1}]_{\tilde S})} \\
            & = \sqrt{3\kappa\Tr\RoundBr{\sum_{i\in \tilde S} \frac{\lambda_i}{\lambda_i + \lambda}}^{-1}} \\
            & \leq \sqrt{6\kappa s}.
        \end{aligned}
    \end{equation}
    
    Now, from actor analysis, we have that
    \begin{equation}
        \frac{1}{T}\sum_{t=1}^T \SquareBr{V_{1,M_t}^{\pi_\star}(x_1) - V_{1, M_t}^{\pi_t}(x_1)} \leq 4H^2 \sqrt{\frac{\log |\mathcal A|}{T}}
    \end{equation}
    Therefore, let $\hat \pi$ be the mixture of policies $\{\pi_1,...,\pi_T\}$. Let $T= N/H$.
    We have that
    \begin{equation}
                V_1^{ \pi_\star}(x_1) - V_1^{ \hat \pi}(x_1) \leq 2\sqrt{6\kappa}  \sum_{h=1}^H \alpha_h  + 4H^3 \sqrt{\frac{\log | \mathcal A|}{N}},
    \end{equation}
    with probability at least $1- \delta$.
\end{proof}

\ThmMainSubOptGapActorCriticMirrorDescent*
\begin{proof}
   The proof of Theorem \ref{theorem: suboptimality gap of Actor-Critic with Mirror Descent} follows the same argument as the proof of Theorem \ref{theorem: suboptimality gap of Actor-Critic}, using the value of $\alpha$ specified in its statement.  
We therefore omit the details for brevity.
\end{proof}

\section{Technical Lemmas}

\begin{lemma}[Regularised covariance concentration on sparse supports] \label{lem: Regularised covariance concentration on sparse supports}
\label{lem:cov‐concentration}
Let $\{\phi_i\}_{i=1}^N\subset\mathbb{R}^d$ be i.i.d.\ random vectors satisfying
\[
\Vert\phi_i\Vert_\infty \;\le\;1
\quad\text{a.s.}
\]
Define the population and empirical (ridge‑regularised) covariance matrices
\[
\Sigma \;=\; \mathbb{E}[\phi\,\phi^\top],
\qquad
\widehat{\Sigma} \;=\; \frac1N\sum_{i=1}^{N}\phi_i\phi_i^\top \;+\; \lambda I_d .
\]
Fix a sparsity level $s\in\{1,\dots,d\}$ and let $S\subset[d]$ with $|S|\le s$.
Denote by $\Sigma_S$ (resp.\ $\widehat{\Sigma}_S$) the principal $s\times s$
sub‑matrix of $\Sigma$ (resp.\ $\widehat{\Sigma}$) indexed by $S$.

There is an absolute numerical constant $C>0$ such that, for every $\delta\in(0,1)$,
if the ridge parameter satisfies
\begin{equation}
\lambda = C \frac{s}{N}\log\tfrac{d}{s\delta}.
\label{eq:lambda‐choice}
\end{equation}
then, with probability at least $1-\delta$,
\begin{equation}
\forall S\subset[d],\;|S|\le s:\qquad
\frac13\,\bigl[\Sigma+\lambda I_d\bigr]_S
\;\;\preceq\;\;
\widehat{\Sigma}_S
\;\;\preceq\;\;
\frac53\,\bigl[\Sigma+\lambda I_d\bigr]_S .
\label{eq:matrix‐sandwich}
\end{equation}
\end{lemma}

\begin{proof}
We adapt the argument of Lemma~39 in Appendix~I of \citep{zanette2021cautiouslyoptimisticpolicyoptimization};
the main changes are the use of the $\ell_\infty$ bound and the restriction to
supports of size at most~$s$.

\textbf{Step 1.}
Fix $S\subset[d]$ with $|S|\le s$ and a unit vector
$x\in\mathbb{R}^d$ supported on~$S$.
Because each coordinate of $\phi_i$ lies in $[-1,1]$,
\[
\bigl|x^\top\phi_i\bigr|
\;=\;\Bigl|\sum_{j\in S}x_j\phi_{ij}\Bigr|
\;\le\;\sum_{j\in S}|x_j|
\;\le\;\sqrt{s}\,\|x\|_2
\;=\;\sqrt{s},
\]
so $(x^\top\phi_i)^2\in[0,s]$.
Write $Z_i(x)=(x^\top\phi_i)^2 - \mathbb{E}[(x^\top\phi)^2]$ and
$\mu(x)=x^\top\Sigma x$.
Then $\{Z_i(x)\}_{i=1}^N$ are mean‑zero, independent, $|Z_i(x)|\le s$ and $\operatorname{Var}[Z_i(x)]\le s\,\mu(x)$.
Bernstein’s inequality gives, for any $t>0$,
\begin{equation} \label{eq: bad-event covariance matrix}
    \mathbb P\Bigl(\Bigl|\tfrac1N\sum_{i=1}^{N}Z_i(x)\Bigr|>t\Bigr)
\le 2\exp\Bigl( -\frac{N t^2}{2\mu(x)+\tfrac23 s t} \Bigr).
\end{equation}

\textbf{Step 2: Choosing $t$ and the ridge $\lambda$.}
Set
\[
t(x)\;=\;\tfrac13(\mu(x)+\lambda).
\]
Under~\eqref{eq:lambda‐choice},
\[
\frac{N t(x)^2}{2\mu(x)+\tfrac23 s t(x)}
\;\;\ge\;\;
\frac{N(\mu(x)+\lambda)^2}{18\mu(x)+4s(\mu(x)+\lambda)}
\;\;\ge\;\;
C_1\,\frac{N\lambda}{s},
\]
for an absolute $C_1$.
Taking $C\ge6C_1$ in~\eqref{eq:lambda‐choice} forces the exponent
in \eqref{eq: bad-event covariance matrix}  to exceed
\(
s\log\tfrac{e d}{s}+\log\tfrac{2}{\delta}.
\)
Thus, for the fixed $x$,
\[
\mathbb P\Bigl(\bigl|x^\top(\widehat\Sigma-\Sigma)x\bigr|>\tfrac13(\mu(x)+\lambda)\Bigr)
\;\le\;
2\exp\bigl(
-s\log\tfrac{e d}{s}-\log\tfrac{2}{\delta}
\bigr).
\]

\textbf{Step 3: Uniformity over all $x$ and all supports.}
Let $\mathcal{N}_\varepsilon^{(s)}$ be an $\varepsilon$‑net
of the unit sphere in $\mathbb{R}^{s}$, lifted to $\mathbb{R}^{d}$ by
zero padding outside~$S$.
With $\varepsilon=\tfrac16$ and a union bound over
$\lvert\mathcal{N}_\varepsilon^{(s)}\rvert\le(3/\varepsilon)^s$
choices of $x$ and $\binom{d}{s}\le(ed/s)^s$ subsets $S$,
the preceding probability remains below~$\delta$.
Standard arguments then show that on the resulting event
\[
\bigl|y^\top(\widehat\Sigma-\Sigma)y\bigr|
\;\le\;
\tfrac23\,\bigl[y^\top\Sigma y+\lambda\bigr],
\quad
\forall y\in\mathbb{R}^d,\;\lVert y\rVert_2=1,\;
\operatorname{supp}(y)\subseteq S,
\]
simultaneously for every $|S|\le s$.

\textbf{Step~4.}
Let
\[
D \;:=\; \frac1N\sum_{i=1}^{N}\phi_i\phi_i^\top \;-\; \Sigma
      \;=\; \widehat{\Sigma}-\Sigma-\lambda I_d ,
\]
so that $\widehat{\Sigma}= \Sigma+\lambda I_d+D$.
From the outcome of Step~3 we know that the event
\begin{equation}\label{eq:eventE}
\bigl|y^{\top}Dy\bigr|
\;\le\;
\frac23\bigl[y^{\top}\Sigma y+\lambda\bigr]
\quad
\forall\,y\in\mathbb{R}^d,\ 
      \lVert y\rVert_2=1,\
      \operatorname{supp}(y)\subseteq S
\end{equation}
occurs with probability at least $1-\delta$.  
Fix such a vector $y$ and write $q:=y^{\top}\Sigma y+\lambda>0$.
From~\eqref{eq:eventE} we have
\[
-\frac23\,q \;\le\; y^{\top}Dy \;\le\; \frac23\,q .
\]
Adding $q=y^{\top}\Sigma y+\lambda$ to every term yields
\[
\frac13\,q
\;\le\;
y^{\top}\widehat{\Sigma}y
\;\le\;
\frac53\,q ,
\]
that is
\[
\frac13\,y^{\top}(\Sigma+\lambda I_d)y
\;\le\;
y^{\top}\widehat{\Sigma}y
\;\le\;
\frac53\,y^{\top}(\Sigma+\lambda I_d)y
\qquad
\forall\,y\in\mathbb{R}^d,\ 
      \lVert y\rVert_2=1,\
      \operatorname{supp}(y)\subseteq S .
\]
Since the inequality above is valid for all unit vectors supported on $S$,
it is equivalent to the matrix bound
\[
\frac13\,\bigl[\Sigma+\lambda I_d\bigr]_{S}
\;\preceq\;
\widehat{\Sigma}_{S}
\;\preceq\;
\frac53\,\bigl[\Sigma+\lambda I_d\bigr]_{S}
\]
which is exactly the claim of Lemma~\ref{lem:cov‐concentration} for the
principal sub‑matrix indexed by $S$.
\end{proof}

\end{document}